%% file: main.tex
\documentclass[11pt, letterpaper, copyright, gdm]{google}

\usepackage[authoryear, sort&compress, round]{natbib}
\input{content/00_meta}

\uselogo{}

\title{Using Reward Uncertainty to Induce Diverse Behaviour in Reinforcement Learning}

\correspondingauthor{anthony.gx.chen@nyu.edu}

\reportnumber{} %

\renewcommand{\today}{2026-09-08}

\author[1,*,+]{Anthony GX-Chen}
\author[2,*]{Ankit Anand}
\author[2,*]{Gheorghe Comanici}
\author[2]{Zaheer Abbas}
\author[2]{Eser Aygün}
\author[2]{David Smalling}
\author[2]{Shibl Mourad}
\author[2]{Doina Precup}
\author[2,*]{Andr{\'e} Barreto}
\author[2,*]{Mark Rowland}

\affil[1]{New York University}
\affil[2]{Google DeepMind}
\affil[+]{Work done as a student researcher at Google DeepMind} 
\affil[*]{Core Contributor}

\begin{abstract}
    \input{content/01_abstract}

\end{abstract}

\begin{document}

\maketitle

\input{content/02_body}

\bibliography{main}

\clearpage

\appendix

\input{content/03_appendix}

\end{document}

%% file: content/00_meta.tex
\usepackage{graphicx}
\usepackage{caption}
\usepackage{subcaption}

\usepackage{amsmath}
\usepackage{amssymb}
\usepackage{mathtools}
\usepackage{amsthm}
\usepackage{mathrsfs}
\usepackage{soul}

\usepackage{thm-restate}

\usepackage{bbm}

\usepackage{enumitem}
\setlist[itemize]{leftmargin=0.5cm,topsep=-2pt,itemsep=-2pt}

\usepackage{tcolorbox}  %

\usepackage{xspace}
\usepackage{makecell}
\usepackage{array}

\usepackage{tabularx}
\usepackage{booktabs} 
\usepackage{multicol}

\usepackage{tikz}
\usetikzlibrary{arrows.meta,positioning,fit,calc}

\usepackage{algorithm}
\usepackage{algorithmic}

\usepackage{wrapfig}

\usepackage{pifont}
\definecolor{darkgreen}{HTML}{308c39}
\definecolor{amber}{HTML}{B8860B}
\newcommand{\greentick}{\textcolor{darkgreen}{\checkmark}}
\newcommand{\redcross}{\textcolor{red}{\ding{55}}}

\newcommand{\methodnamelong}{Randomized Objectives, Set Actions\xspace} 
\newcommand{\methodnameshort}{ROSA\xspace} 
\xspaceaddexceptions{+}  %

\newcommand{\iidsim}{\overset{\text{i.i.d.}}{\sim}}
 
\newcommand{\ans}{a} 
\newcommand{\ansset}{{\mathcal{A}^*}}

\newcommand{\setY}{\mathbf{Y}}

\newcommand{\ind}{\mathbf{1}}
\newcommand{\indf}[1]{\ind_{[#1]}}

\newif\ifshowcomments

\showcommentsfalse

\ifshowcomments
    \newcommand{\todo}[1]{\textcolor{red}{\textbf{TODO:} #1}}
    \newcommand{\markcomment}[1]{\textcolor{blue}{[Mark: #1]}}
    \newcommand{\commentoptional}[1]{}  %
    \newcommand{\antcomment}[1]{\textcolor{purple}{[Ant: #1]}}
    \newcommand{\ghecomment}[1]{\textcolor{brown}{[Gheorghe: #1]}}
    \newcommand{\andrecomment}[1]{\textcolor{red}{[Andre: #1]}}
    \newcommand{\ankitcomment}[1]{\textcolor{teal}{[Ankit: #1]}}
\else
    \newcommand{\todo}[1]{}
    \newcommand{\markcomment}[1]{}
    \newcommand{\antcomment}[1]{}
    \newcommand{\commentoptional}[1]{}  %
    \newcommand{\ghecomment}[1]{}
    \newcommand{\andrecomment}[1]{}
    \newcommand{\ankitcomment}[1]{}
\fi

\newcommand{\sneakynegvspace}{\vspace{0em}}

\definecolor{citationcolor}{RGB}{0,91,150}  
\hypersetup{
    colorlinks=true,
    linkcolor=citationcolor,
    citecolor=citationcolor,
    filecolor=citationcolor,
    urlcolor=citationcolor,
}

\theoremstyle{plain}
\newtheorem{theorem}{Theorem}[section]

\newtheorem{lemma}[theorem]{Lemma}

\theoremstyle{definition}

\theoremstyle{remark}
\newtheorem{remark}[theorem]{Remark}

%% file: content/01_abstract.tex
Classical reinforcement learning (RL) typically seeks a deterministic policy that maximizes the expected sum of a scalar reward. Yet, modern applications such as language model fine-tuning or scientific discovery demand diversity. Existing remedies such as entropy regularization or diversity bonuses often require fragile trade-offs that sacrifice performance for stochasticity or rely on heuristic metrics that can misalign policy rankings. 
We argue that diversity is more naturally understood as the rational response to \textit{uncertainty} in the reward. When the reward function is not perfectly known---as is the case with ambiguous preferences or imperfect reward models---committing to a single action can be sub-optimal. Building on this, we propose a fundamental reformulation of the RL objective by replacing the scalar reward with a distribution over reward functions, and applying a non-linear objective over \textit{sets} of actions. The result is a framework in which \textit{calibrated} behavioural diversity emerges naturally, remains controllable through the reward function distribution, and is obtained without sacrificing expected reward. Focusing on the contextual bandit setting as commonly used in large language model (LLM) post-training, we derive a principled gradient estimator for this objective and prove that our formulation naturally generalizes both vanilla policy gradient and more recently developed action-set approaches. We provide didactic experiments which complement our theoretical results, and our large-scale empirical results in LLM reasoning further demonstrate that this framework offers a robust and theoretically grounded alternative for complex RL tasks where the traditional formulation of the problem fails to induce the desired breadth of agent behaviour.

%% file: content/02_body.tex
\section{Introduction}\label{sec:intro}

Classical reinforcement learning (RL) theory rests on the foundational assumption that preferences can be captured by the expected sum of a scalar reward function. A well-known consequence of this assumption is the existence of an optimal deterministic policy \citep{puterman94mdp}, and, accordingly, most standard RL methods are designed to converge to such a solution. However, as RL is increasingly applied to open-ended domains, convergence to a single behaviour has become a significant limitation. In settings such as the fine-tuning of large language models (LLMs), it is often essential to maintain the diversity of the pre-trained policy in order to preserve its usefulness for creative tasks and chain-of-thought reasoning \citep{kirk2023understanding,chung2025modifying,ismayilzada2025creative,cui2025entropy,west2025base,zhao2025echo,shypula2025evaluating,yang2025alignment,yun2025price}. Beyond language, the capacity for inference-time exploration is vital for scientific discovery, where agents must navigate vast search spaces to identify rare, high-value solutions \citep{romera-paredes2024, Hubert2025,aygun2025aihelpscientistswrite}. Furthermore, convergence to a deterministic optimal policy is detrimental when the reward model itself is only an imperfect proxy for some true underlying preference, and over-optimization results in the deterioration of generation quality \citep{gao2022rm,coste2023reward,moskovitz2023confronting,rafailov2024scaling,lambert2026overoptimization}.

Existing approaches for inducing diversity include using entropy regularization or diversity bonuses. Both families of approaches can be thought of as modifications to the original reward function, resulting in policies that are highly sensitive to the explicit choice of diversity measures, and must be carefully tuned. A variety of approaches make use of policy entropy penalties \citep{williams1991function,mnih2016asynchronous,todorov2006linearly,ziebart2008maximum,haarnoja2017reinforcement}, 
although this forces a trade-off, in which a more stochastic policy can be obtained at the cost of reduced expected rewards \citep{jhaveri2025convergence}. Diversity bonuses that get added to the rewards similarly introduce a tension between optimizing the original reward function and focusing on the bonus \citep{li2025jointly,hamid2025polychromic,orney2026poly}, and may result in undesirable ranking of sub-optimal policies, as we discuss further below.

In this work, we propose a shift in perspective: diversity should be understood as the rational response to \textit{reward uncertainty}. For instance, in the case of generating diverse, creative model outputs, the uncertainty arises from ambiguity in the user's preferences, which are often under-specified by the input prompt alone. In the context of mitigating reward model over-optimization, this uncertainty reflects an epistemic gap concerning the definitive reward function. In each of these cases, the learning objective is more holistically described by considering a \textit{distribution} over reward functions, rather than by augmenting a single reward function. 

However, standard policy gradient updates over a distribution of reward functions often collapse the policy to average, non-specialized behaviour. In response to this, we propose to use policy gradients that depend non-linearly on \emph{sets of actions} \citep{koyamada2024emergence,tang2025optimizing,hamid2025polychromic} and sampled reward functions. While this approach has previously been explored in the Markov decision process (MDP) setting to promote exploration \citep{koyamada2024emergence}, we provide an in-depth analysis 
as applied to generative modelling, showing it allows the policy to flexibly learn diverse behaviours corresponding to different regions of the distribution of reward functions, without collapsing behaviour onto optimal actions for the most frequently encountered reward function.
Concretely, focusing on the contextual bandit setting, we derive a principled gradient estimator for this objective and prove that our formulation naturally generalizes both vanilla policy gradient and recently developed approaches based on action sets. Our empirical results demonstrate that this framework provides a more robust and natural foundation for complex RL settings, where simple reward maximization fails to induce the desired breadth of agent behaviour.

Our core contributions are as follows:
\begin{itemize}
    \item
    We introduce a new family of RL objectives that provides calibrated diversity control through a distribution over reward functions, which we refer to as \textit{\methodnamelong} (\methodnameshort) (Sec.~\ref{sec:diverse-preferences}).
    \item We provide an analysis of this new family of objectives, both in terms of the global optimizing policies and the landscape of the objectives themselves. This simultaneously provides theoretical guarantees for the objectives and informs useful hyperparameter settings in practice (Sec.~\ref{sec:theory}).
    \item We complement this analysis with a variety of empirical studies, both in toy domains, to develop a fine-grained understanding of the method's performance, and in larger-scale domains, to demonstrate the broad applicability of the proposed objectives in practical scenarios (Sec.~\ref{sec:experiments}).
\end{itemize}

\section{Background}
\label{sec:background}

We consider a contextual bandit scenario \citep{langford2007epoch,sutton2018reinforcement} defined over action space~$\mathcal{Y}$ and state space~$\mathcal{X}$. This encompasses both abstract bandit problems, as well as large-scale generative modelling such as single-turn fine-tuning of language models. A typical RL problem is specified via a reward function $R : \mathcal{X} \times \mathcal{Y} \rightarrow \mathbb{R}$, and the goal is to optimize a policy $\pi : \mathcal{X} \rightarrow \mathscr{P}(\mathcal{Y})$ according to the \textbf{mean reward criterion} for some state distribution $\mu$,
\begin{align}
    \mathcal{J}(\pi) = \mathbb{E}_{X \sim \mu, Y \sim \pi(\cdot|X)}\big[R(X, Y)\big] \, .
    \label{eq:objective}
\end{align}

Policy gradient (PG) methods are common means of optimizing the policy $\pi$ with respect to the objective in Eq.~\eqref{eq:objective}. The prototypical PG method is REINFORCE~\citep{williams1992simple}; given a parametrized policy $\pi_\theta$ and a state--action sample $(X, Y)$ drawn according to $\pi_\theta$, the algorithm updates the parameters $\theta$ in the direction
\begin{align}\label{eq:reinforce}
    \nabla_\theta \mathcal{J}_{\text{PG}}(\pi_\theta) = \mathbb{E}_{X \sim \mu, Y \sim \pi_\theta(\cdot|X)} \big[ R(X, Y) \nabla_\theta \log \pi_\theta(Y|X) \big] \, .
\end{align}

\subsection{Inducing diversity with policy gradients}\label{sec:background-other-methods}

As described in Section~\ref{sec:intro}, there are many applications of reinforcement learning where we want to maintain some level of diverse behaviour after training, rather than collapsing to a deterministic policy. Before introducing the core algorithmic proposal of the paper, we briefly review several existing families of approaches that aim to achieve this.

\emph{Entropy regularization} \citep{williams1991function,todorov2006linearly,ziebart2008maximum,mnih2016asynchronous,haarnoja2017reinforcement} applies a (scaled) entropy bonus $\mathcal{H}(\pi(\cdot|x)) = -\sum_{y \in \mathcal{Y}}\pi(y|x) \log( \pi(y|x))$ to the basic objective in Eq.~\eqref{eq:objective}. This prevents policy gradient updates from reaching a deterministic policy, and the optimal policy is guaranteed to place non-zero mass on all actions.

\emph{Diversity bonuses} introduce an additional diversity function, quantifying a level of diversity among a given collection of actions. Typically, the integrand in Eq.~\eqref{eq:objective} is then modified by considering a sequence of sampled actions $Y_{1:n}$ at a given state $X$, and modifying rewards based on the diversity function. This can be implemented on a per-sample basis \citep[see, e.g.,][]{li2025jointly}, with a pairwise diversity function $d : \mathcal{Y}^2 \rightarrow \mathbb{R}$ and modified reward $R(X, Y_i) \sum_{j \not= i} d(Y_i, Y_j)$, or at a group level \citep{hamid2025polychromic}, with a group diversity function $D : \mathcal{Y}^n \rightarrow [0, \infty)$, and modified reward $\tfrac{1}{n} \sum_{i=1}^n R(X, Y_i) D(Y_{1:n})$.

\emph{Multi-objective reinforcement learning} \citep{roijers2013survey,hayes2022practical} specifies a finite collection of reward functions $R_1,\ldots,R_d$, and then optimizes a scalarized objective. This is typically done via \emph{scalarized expected returns} (SER) or \emph{expected scalarized returns} (ESR), taking the form $s( \mathbb{E}_{Y \sim \pi}[\mathbf{R}])$, or, respectively, $\mathbb{E}_{Y \sim \pi}[s(\mathbf{R})]$. Here, $\mathbf{R} = (R_1(Y), \ldots, R_d(Y))$ is the vector of rewards, and $s : \mathbb{R}^d \rightarrow \mathbb{R}$ is a (possibly non-linear) scalarization function.

\section{Optimizing behaviour for distributions of reward functions}
\label{sec:diverse-preferences}

Our goal is to develop a policy gradient method that learns diverse behaviours capable of reflecting three key criteria: (a) uncertainty in the objective function; (b) variance in human judgment regarding what constitutes a high-quality action (e.g., scenarios where a single query yields multiple valid answers); and (c) a strategic desire to maintain policy flexibility for future fine-tuning, such as adapting to individual user preferences. While the existing approaches outlined in Section~\ref{sec:background-other-methods} serve foundational algorithmic roles in reinforcement learning, they exhibit several shortcomings when applied to the specific use cases described above. We begin by analyzing these limitations in detail (further discussed in Appendix~\ref{app:related-work} and summarized in Table~\ref{tab:desiderata-comparison}).

\textbf{Inclusion of poor-quality actions in the optimal policy.} The optimal policy for the entropy-regularized objective puts non-zero mass on all actions, including ones with very bad rewards. This is also true for some variants of diversity-bonus methods.

\textbf{Undesirable ordering of sub-optimal policies. } Many approaches induce undesirable orderings of sub-optimal policies: a policy with lower average reward (but higher diversity) can be preferred over a policy with higher average reward.

\textbf{Inability to express a stochastic optimal policy. } A number of objectives, including vanilla policy gradient, pass-at-$k$/best-of-$n$ gradient, and the expected scalarized return (ESR) always have an optimal deterministic policy, and so cannot be guaranteed to induce diversity.

\textbf{Unavailability of unbiased gradients via direct sampling. } The scalarized expected return (SER) multi-objective RL formulation typically is difficult to optimize, due to the unavailability of unbiased sample-based gradient estimators of nonlinear objectives over expectations, $s(\mathbb{E}[\cdot])$.

In summary, for our problem of learning a policy that reflects diverse behaviour under the motivations outlined above, these existing approaches suffer several drawbacks: inexpressivity of the objective, incorrect optimal/sub-optimal policy orderings, and optimization problems.

\begin{figure}[htb]
    \centering 
    \begin{subfigure}[t]{0.46\linewidth}
        \centering
        \includegraphics[width=\textwidth]{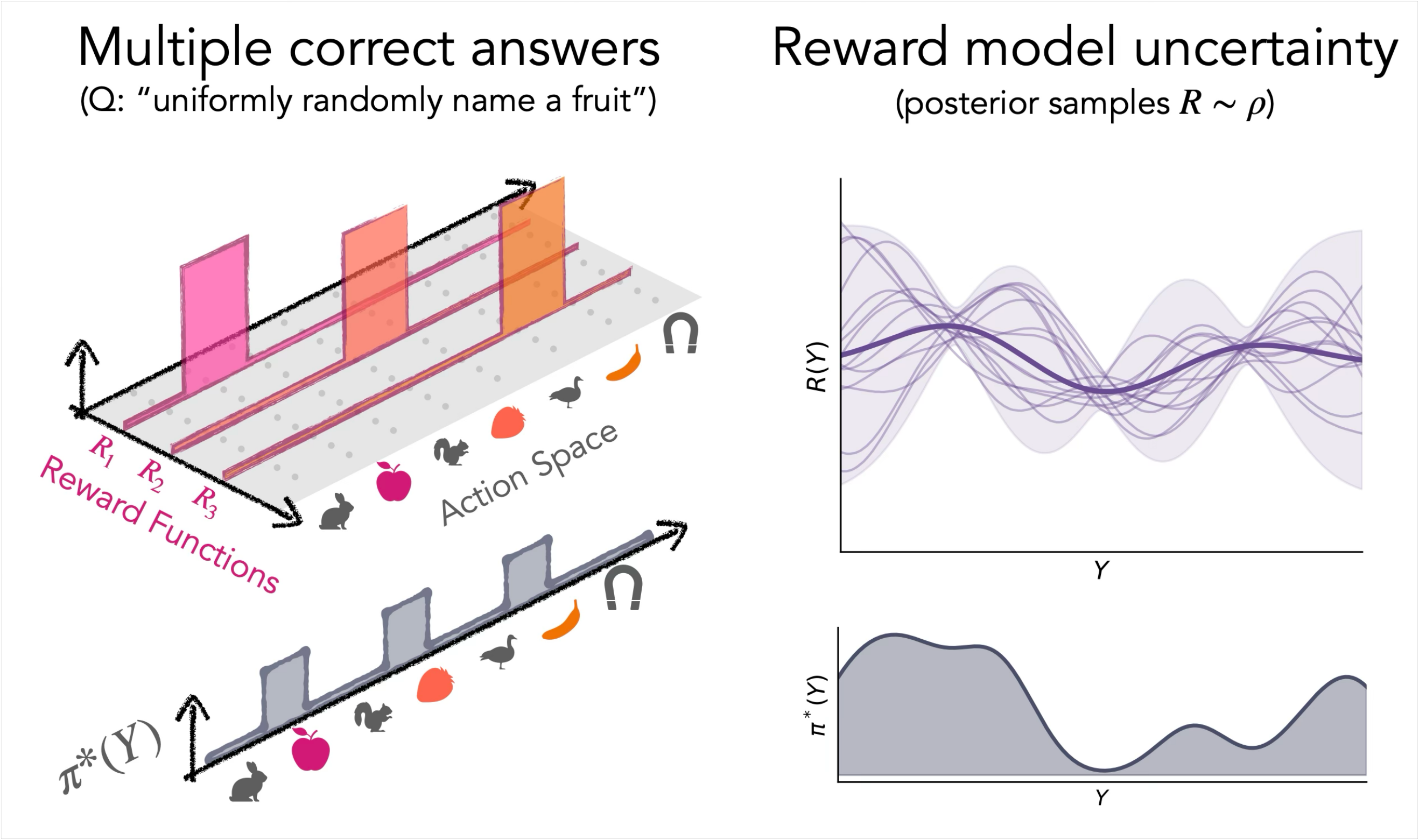}
        \caption{Examples of distributions of reward functions and the desired policy. \textbf{(Left)} A binary reward setting with a variety of correct actions; we prefer a policy that places probability mass on all correct actions. \textbf{(Right)} We only have access to approximate reward functions ($\hat{R}$'s, shaded lines). There is epistemic uncertainty about the value of the true underlying reward (bold line), and we wish to avoid overfitting to the approximate $\hat{R}$.} 
        \label{fig:example-tasks}
    \end{subfigure}
    \hfill
    \begin{subfigure}[t]{0.53\linewidth}
        \centering
        \includegraphics[width=\textwidth]{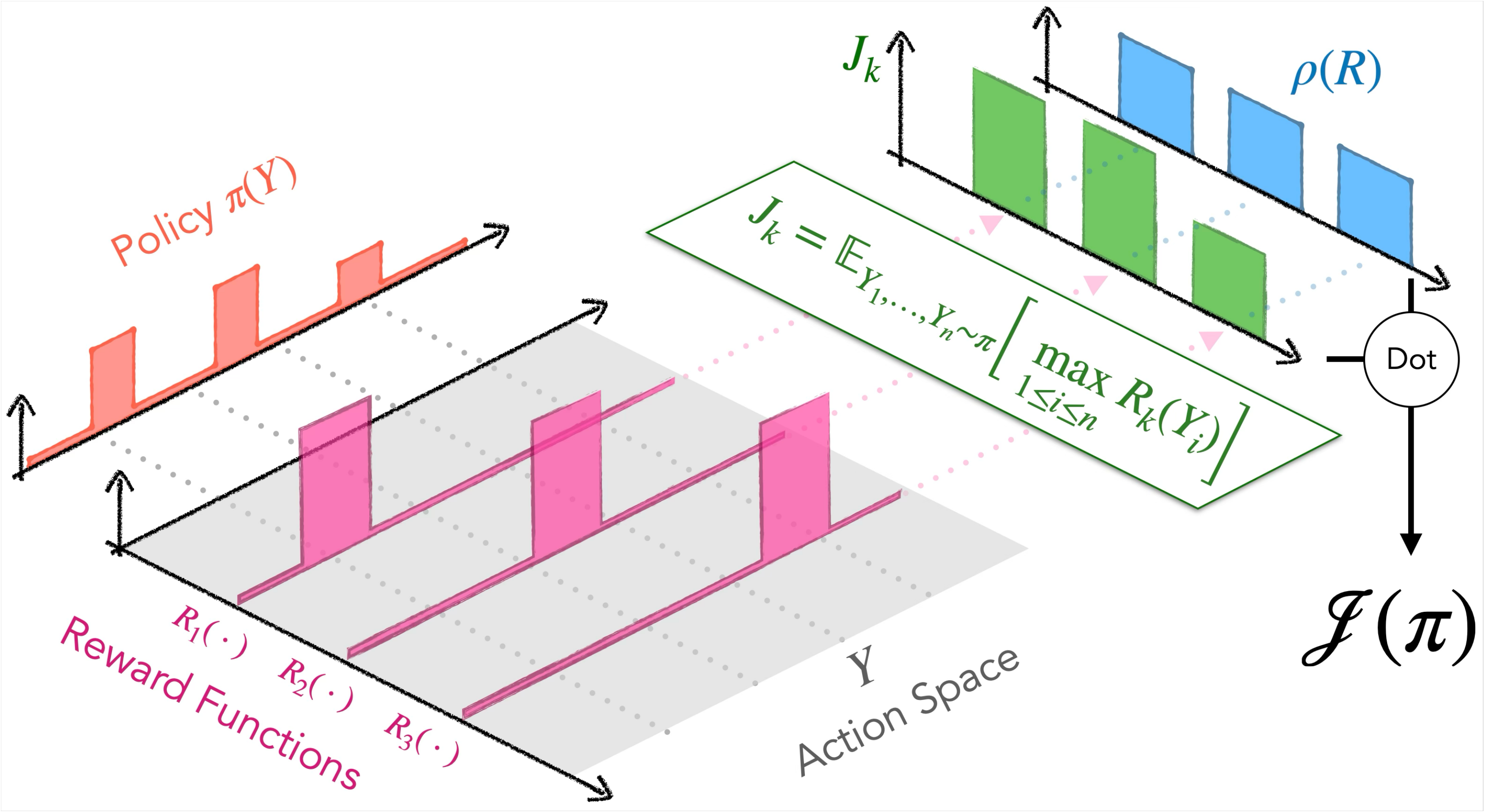}
        \caption{\methodnameshort+Max optimizing a policy for a distribution of reward functions). Given some policy (orange) and reward functions $R_k$ (pink), we first compute the policy's max-of-$n$ performance for \textit{each} reward function, $J_k$ (green). We then take the dot product with the reward function probabilities ($\rho(R_k)$, blue) to get the \methodnameshort+Max score. Optimizing this score calibrates the policy to the reward function uncertainties. Appendix~\ref{app:related-work} contains more details and comparisons with regular PG.
        }
        \label{fig:rosa-max-computation-flow}
    \end{subfigure}
    \caption{Many complex tasks are naturally expressed as distributions over reward functions.}
    \label{fig:rosa-applications}
    \sneakynegvspace
    \vspace{-0.1in}
\end{figure}

\subsection{The Randomized Objectives, Set Actions (\methodnameshort) criterion}
\label{sec:rosa-objective}

We now make our core algorithmic proposal, aiming to close the drawbacks described in the previous section. We build on the approach of working with multiple reward functions that express our desired diverse behaviours, and develop an approach that (i) naturally allows for arbitrary distributions over reward functions, (ii) avoids inducing learnt policies to take bad actions solely for the sake of increased diversity, (iii) has straightforwardly-implementable, unbiased policy gradient estimates. The key idea of our proposal is to generalize Eq.~\eqref{eq:objective} by considering (i) \textit{multiple} actions sampled i.i.d.\ from the policy $\pi$ \citep{tang2025optimizing,hamid2025polychromic}; and 
(ii) a \textit{distribution} $\rho$ over reward functions $R$. This gives rise to the family of \textit{\textbf{R}andomized \textbf{O}bjectives, \textbf{S}et \textbf{A}ctions} (\textbf{\methodnameshort}) criteria.

Formally, let $\rho \in \mathscr{P}(\mathbb{R}^{\mathcal{X} \times \mathcal{Y}})$ be a distribution over reward functions, and let $\setY = (Y_i)_{i=1}^n$ denote a \textit{multiset} of actions induced by $Y_1, \dots, Y_n$. We use tuple notation to  signify multiplicities are retained, but any orders are ignored. We use the shorthand $\setY \sim \pi_\theta$ to denote the $n$ actions are sampled i.i.d.\ from policy $\pi_\theta$. Let $R \sim \rho$ be a sample from the distribution over reward functions. For brevity, we denote the reward multiset as $R(\setY) = (R(Y_i))_{i=1}^n$. 
While our general objective makes use of a set function $f: \mathbb{R}^n \rightarrow \mathbb{R}$ (general case in Section~\ref{sec:rosa-general-f}), we begin by studying the $f = \max$ set function,
\begin{equation}
    \mathcal{J}_{\text{\methodnameshort+Max}}(\pi_\theta) = \mathbb{E}_{X \sim \mu \,,\,  \setY \iidsim \pi_\theta (\cdot|X)}\bigg[ \mathbb{E}_{R \sim \rho} \Big[
        \max_{1 \leq i \leq n} R(X, Y_i)
    \Big] \bigg] \,,
    \label{eq:rosa-max-criterion}
\end{equation}
where the $\max$ function aggregates over the reward of the $n$ actions in a permutation-invariant way for each reward function. Figure~\ref{fig:rosa-max-computation-flow} contains an example in the case of a discrete set of reward functions.

Intuitively, the \methodnameshort+Max criterion provides an \emph{asymmetric} summary of performance on each sampled reward function, encouraging \emph{at least one} action in the multiset to perform well in each case. This is related to the expected utilities used in multi-objective RL, described above, but differs in several important ways: (i) the form of optimism used is strictly more expressive than that typically used in MORL (see Appendix~\ref{app:related-work}) leading to more desirable optimal policies, (ii) the objective has a straightforwardly derivable unbiased policy gradient estimator (see below), and (iii) the objective naturally supports arbitrary distributions over reward functions (rather than an explicitly declared finite number of reward functions). 
Additionally, \citet{koyamada2024emergence} and \citet{nishimori2026emergence} propose the ReMax objective which coincides with \methodnameshort+Max in the bandit setting and is similarly motivated by an observation about the induced optimal policy becoming stochastic. 
They focus on exploration algorithms for MDPs without explicit reward/value distributions, rather than policy diversity for generative modelling (see Section~\ref{sec:related-work} for further discussion). As we will demonstrate in Section~\ref{sec:analysis-rosa-max}, a key consequence of the \methodnameshort+Max objective is that its optimal policy assigns action probabilities over the set of actions that are optimal under \textit{each} $R \sim \rho$. Finally, note that the distribution over \emph{reward functions} in \methodnameshort differs from the distribution over \emph{rewards / returns} modelled in distributional RL \citep{chung1987discounted,morimura2010nonparametric,bellemare2023distributional}; we discuss their relationship in Appendix~\ref{app:risk-sensitive-rl}.

We now derive an unbiased, variance-reduced policy gradient estimator to optimize our proposed objective. Given $m$ sampled reward functions $(R_k)_{k=1}^m$ from $\rho$, and $n$ sampled actions $(Y_i)_{i=1}^n$ from $\pi_\theta (\cdot | X)$, an unbiased gradient estimator of Equation~\eqref{eq:rosa-max-criterion} is given by:
\begin{equation}
    \hat{g}(\theta) = \frac{1}{m} \sum_{k=1}^m \bigg[
        \sum_{i=1}^n \Big(
            \max_{1\leq j \leq n} R_k(X, Y_j) \, - \max_{\substack{1\leq j\leq n\,,\\j\neq i}} R_k(X, Y_j)
        \Big) \nabla_\theta \log \pi_\theta (Y_i | X)
    \bigg] \,,
    \label{eq:rosa-max-loo-pg}
\end{equation}
which has the effect of only the maximal reward action contributing a non-zero term to the estimator. Note the term $\max_{1\leq j\leq n \,,\, j\neq i} R_k(Y_j)$ is an optional control variate for variance reduction, which is analogous (in the single reward function setting) to the estimator used in \citet{tang2025optimizing}. %
We provide derivations for a generalization of the above gradient estimator (for general set functions and control variates) and all other theoretical results in the main paper in Appendix~\ref{sec:proofs}.

\begin{remark}
    While set-function objectives based on best-of-$n$ have been investigated in the single-reward setting as a means of inducing generative diversity (see, e.g., \citealt{tang2025optimizing,walder2026pass}),
    we emphasize such formulations do not induce a stochastic optimal policy---a deterministic optimal policy always exists for the single reward function.
    By contrast, in our setting, the combination of randomized reward functions \emph{and} set function objectives induces an optimal stochastic policy, which balances the incentives of the randomized reward functions.\footnote{This observation, in the case of the $\text{max}$ set function, was also made by \citet{koyamada2024emergence,nishimori2026emergence}, along with an unbiased max-at-$k$ gradient estimator. However, these works do not develop it in the generative modelling setting. See Section~\ref{sec:related-work} for further discussions.}
\end{remark}

\subsection{Theoretical analysis of \methodnameshort+Max}
\label{sec:analysis-rosa-max}

Having introduced the \methodnameshort+Max update in Eq.~\eqref{eq:rosa-max-loo-pg}, we now examine it theoretically. Without loss of generality, we consider the single-state case, and drop state-dependence from our notation. Our theoretical analysis will assume each reward function $R \sim \rho$ is binary and ``one-hot'', assigning reward 1 to a distinct optimal action $y^* \in \mathcal{Y}$ and 0 otherwise. We note that the \methodnameshort objective and gradient estimator support arbitrary reward functions in practice.

\begin{restatable}{proposition}{propOptimalBinaryDistinct}\label{prop:optimal-binary-distinct}
    \textbf{Optimal policy of \methodnameshort+Max with uniform reward function distributions}. 
    Consider $m$ binary reward functions $(r_i)_{i=1}^m$, each with a single distinct optimal action $y^*_i \in \mathcal{Y}$, and set $\rho$ uniform over $(r_i)_{i=1}^m$. Then, writing $\delta_y$ for the Dirac delta distribution at $y$, the \methodnameshort+Max objective with any action set size $n \geq 2$ has unique optimal policy
    \begin{align*}
        \pi^* =  \sum_{i=1}^m \frac{1}{m} \delta_{y_i^*} \, .
    \end{align*}
\end{restatable}

This result verifies that the \methodnameshort+Max  objective achieves our intended goal of having the maximally diverse policy over correct responses at its unique optimizer. To develop intuition further, we numerically simulate the global objective landscape over all policies in an example with two reward functions (Figure~\ref{fig:rosa-uniform-simplex}). We compare \methodnameshort with other diversity-inducing RL objectives, as well as the standard PG objective (Figure~\ref{fig:vanilla-pg-uniform-simplex}), that does not induce any kind of diversity among optimal actions.

Firstly, adding entropy regularization to a standard RL objective \citep{todorov2006linearly,ziebart2008maximum,haarnoja2017reinforcement} changes the ordering of preferred policies to favour non reward-maximizing policies with higher entropy (Figure~\ref{fig:entropy-reg-uniform-simplex}), and disproportionally favour actions with minor improvements in rewards \citep{gx2025kl}. We also consider diversity bonuses \citep{li2025jointly,hamid2025polychromic,orney2026poly}, such as  in the form of a multiplicative diversity bonus, e.g., $\bar{r}(X,Y_i) = r(X,Y_i) \text{Div}(X, \setY)$ where $\text{Div}(X, \setY)$ is the average distance to all other samples in the action multiset. This family of objectives can induce a global optimum near/at a reward maximizing policy, but re-orders sub-optimal policies where a high-diversity, low-reward policy can be favoured over a high-reward policy (Figure~\ref{fig:mult-diversity-uniform-simplex}). Finally, \methodnameshort+Max does not have either of these issues: its gradient points solely in the direction of both higher reward and higher action diversity (Figure~\ref{fig:rosa-max-uniform-simplex}). It should be noted that both ``diversity bonus'' methods and \methodnameshort require more information than the original reward function, while entropy regularization does not; diversity bonus requires a diversity function, while \methodnameshort requires a \textit{distribution} of reward functions.

\subsection{Choosing the action set size}
\label{sec:n-actions-on-hessian}

It is interesting to note that the action sampling parameter $n$ in Prop.~\ref{prop:optimal-binary-distinct}  does \textit{not} affect the optimal policy, meaning that for a uniform reward function distribution, \textit{any} $n\geq 2$ induces a maximally diverse policy over all reward functions; thus, the global optimum can be obtained with small parallel sampling budgets (e.g., just $n=2$ i.i.d.\ action samples from $\pi_\theta$). However, in practice we expect the selection of this parameter to be important for algorithmic performance. To build some intuition for why this is, we can calculate the curvature of the objective at the optimal policy.

\begin{restatable}{lemma}{lemHessian}\label{lemma:rosa-max-hessian}
    Under the assumptions of Proposition~\ref{prop:optimal-binary-distinct}, the Hessian of the objective at the optimal policy is diagonal, with all diagonal elements given by $-n(n-1) (1 - \tfrac{1}{m})^{n-2}$.
\end{restatable}

The term in Lemma~\ref{lemma:rosa-max-hessian} has absolute value of 2 when $n=2$ for all $m$, is maximized (in absolute value) when $n \approx 2m$, and decays toward 0 as $n \gg 2$.
This fits the intuition that a large $n$ results in a flat landscape around the optimum, slowing optimization. It also suggests $2 \leq n < 2m$ as a sensible hyperparameter range to obtain a steep objective landscape around the optimal policy (when $m$ is known). We visualize and discuss this more in Appendix~\ref{app:objective-landscape-slice-rosa-max}.

\begin{figure}[t]
    \centering 
    \begin{subfigure}[b]{0.22\linewidth}
        \centering
        \includegraphics[width=\textwidth]{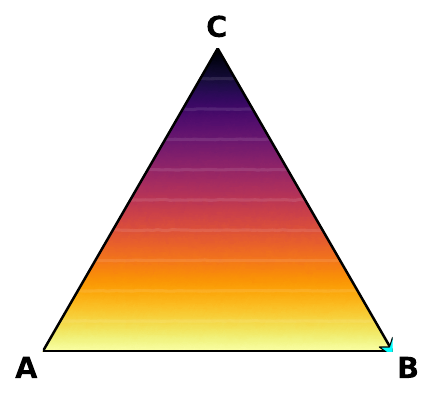}
        \caption{Vanilla PG} 
        \label{fig:vanilla-pg-uniform-simplex}
    \end{subfigure}
    \hfill
    \begin{subfigure}[b]{0.22\linewidth}
        \centering
        \includegraphics[width=\textwidth]{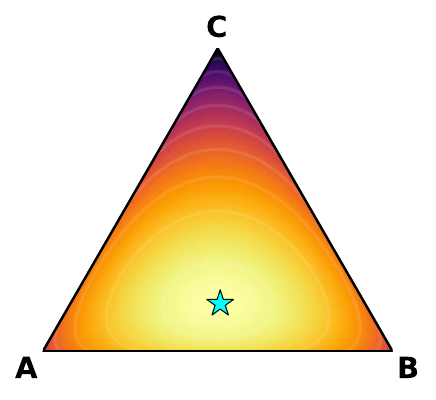}
        \caption{Entropy Reg}
        \label{fig:entropy-reg-uniform-simplex}
    \end{subfigure}
    \hfill
    \begin{subfigure}[b]{0.22\linewidth}
        \centering
        \includegraphics[width=\textwidth]{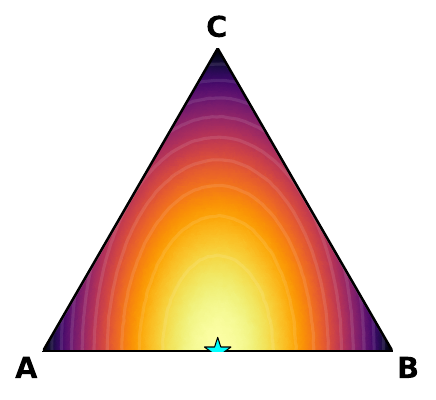}
        \caption{Mult. Diversity}
        \label{fig:mult-diversity-uniform-simplex}
    \end{subfigure}
    \hfill
    \begin{subfigure}[b]{0.22\linewidth}
        \centering
        \includegraphics[width=\textwidth]{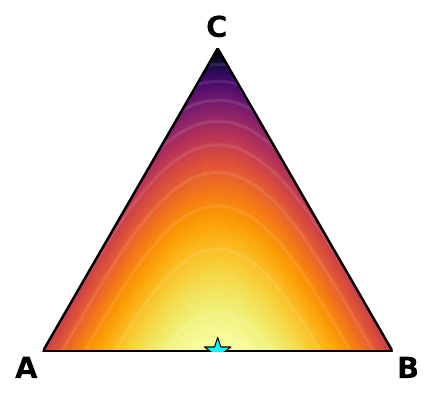}
        \caption{\methodnameshort+Max}
        \label{fig:rosa-max-uniform-simplex}
    \end{subfigure}
    \caption{Simplices of global objective landscapes over 3-action categorical policies for four RL objectives. Actions \texttt{A} and \texttt{B} receive reward +1, and \texttt{C} receives reward 0. Details in Appendix~\ref{app:objective-simplex}.
    }
    \label{fig:rosa-uniform-simplex}
    \vspace{-0.2in}
\end{figure}

\section{Beyond the max: \methodnameshort with general distributions and set functions}
\label{sec:theory}

Having established a result on the optimal solution in the case of uniform distribution over rewards and the $\max$ set function in Proposition~\ref{prop:optimal-binary-distinct}, we now generalize it in two ways: to a general family of set functions, and to non-uniform reward function distributions.

\subsection{A family of set functions with maximally diverse optimal \methodnameshort policies}
\label{sec:rosa-general-f}

So far, we have shown that the \methodnameshort+Max criterion has the desirable property of having an optimal policy that places mass equally over all optimal answers. The reader may wonder if set functions other than ``max'' can be used to aggregate over rewards $(R(Y_i))_{i=1}^n$. Here, we present an additional theoretical result which complements  Proposition~\ref{prop:optimal-binary-distinct} and establishes this optimal policy guarantee for a more general family of set functions.

We focus our analysis on the binary reward setting. Under binary rewards, any general multi-set function $f(R(X, Y_1),\ldots,R(X, Y_n))$ can be simplified into a \textit{success-count reward function} $\tilde{f}(\sum_{i=1}^n R(X, Y_i))$ depending only on the \emph{success count} $\sum_{i=1}^n R(X, Y_i)$; 
because the individual rewards are binary, the function's value is determined entirely by the number of correct actions.

\begin{restatable}{theorem}{thmOptimalBinaryDistinct}\label{thm:optimal-binary-distinct}
    \textbf{Optimal policy for general $f$}. Consider $m$ uniformly distributed (binary, distinct) reward functions $(r_i)_{i=1}^m$, each with a single distinct optimal action $y^*_i \in \mathcal{Y}$, and a set function $f$, with corresponding success-count reward function $\tilde{f}$ (defined above) which is strictly increasing as well as strictly concave in the sum of rewards. Then the \methodnameshort+$f$ objective has a global optimum which samples correct actions from all reward functions uniformly,
    \begin{align*}
        \pi^* = \frac{1}{m} \sum_{i=1}^m \delta_{y_i^*} \, .
    \end{align*}
\end{restatable}

Theorem~\ref{thm:optimal-binary-distinct} allows the construction of many different kinds of set functions which all have the same optimum at the maximally diverse policy, yet can have vastly different optimization properties. One function from this family is the Softmax, which is a natural extension of the max that varies smoothly between $f_{\text{max}}$ and $f_{\text{mean}}$,

\begin{minipage}[c]{.73\textwidth}
    \begin{equation}
        f_{\text{Softmax}}\big(\big(R(X, Y_i)\big)_{i=1}^n\big) = \sum_{i=1}^n \frac{\exp\big(R(X, Y_i)\big)}{\sum_{j=1}^n \exp\big(R(X, Y_j)\big)} R(X, Y_i) \,.
        \label{eq:softmax-set-function}
    \end{equation}
    Numerical simulations in Figure~\ref{fig:rosa-softmax-uniform-simplex} illustrate that, like the Max, the Softmax also has a maximally diverse global optimum, though with different optimization landscape. Notably, the Softmax function also removes the sparsity of the variance-reduced gradient in Eq.~\eqref{eq:rosa-max-loo-pg} in the case of the max set function. We leave exploration of other set functions for future work and focus primarily on $f_\text{Max}$ and $f_\text{Softmax}$ in the next sections. See also \citet{verdun2025soft} for the use of softmax distributions as an alternative to best-of-$n$ sampling in the single-reward setting.
\end{minipage}%
\hfill
\begin{minipage}[c]{.23\textwidth}
    \vspace{-1em}
        \centering
        \includegraphics[width=\linewidth]{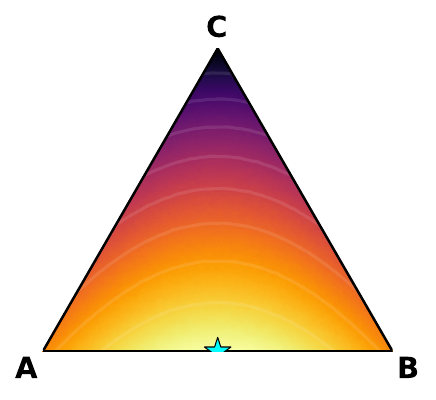}
        \captionof{figure}{{\methodnameshort} + Softmax simplex in same setting as Figure~\ref{fig:rosa-uniform-simplex}.}
        \label{fig:rosa-softmax-uniform-simplex}
        \vspace{-1em}
\end{minipage}%

\subsection{Non-uniform reward function distribution induces controllable diversity}

We now consider non-uniform weightings in the reward function distribution $\rho$. In this setting, the relationship between $\rho$ and optimal policy action probabilities is more complex. We derive a specific result for the \methodnameshort+Max objective and leave the extension to other set functions for future work.

\begin{restatable}{proposition}{propNonUniform}\label{prop:non-uniform}
    Consider $m$ distinct, binary reward functions $(r_i)_{i=1}^m$ (each with a single distinct optimal action $y^*_i \in \mathcal{Y}$) with probabilities $(\alpha_i)_{i=1}^m$ under $\rho$, and assume without loss of generality that $\alpha_1 \geq \dots \geq \alpha_m$. The \methodnameshort+Max objective with $n\geq 2$ action samples has an optimal policy which samples optimal actions from the top-k reward functions (ordered by $\alpha$'s), each with probability,
    \begin{equation}
        p_i^* = 1 - \frac{(k-1) \, \alpha_i^{-1/(n-1)}}{\sum_{j=1}^k \alpha_j^{-1/(n-1)}} \,,
    \end{equation}
    where $k$ is the largest integer between $1$ and $m$ where $p_k^*$ is positive.
\end{restatable}

\begin{wrapfigure}{r}{.53\textwidth}
\begin{minipage}{\linewidth}
    \rule{\linewidth}{0.8pt}
    \vspace{-17pt}
    \captionof{algorithm}{\methodnameshort+Max single-state policy update}
    \label{alg:rosa-loo}
    \vspace{-6pt}
    \rule{\linewidth}{0.8pt}
    \vspace{-12pt}
    \begin{algorithmic}[1]
    \REQUIRE Policy $\pi_\theta$, state $X$, reward functions $\{R_k\}_{k=1}^m$ each with probabilities $\{\alpha_k\}_{k=1}^m$.
    \STATE Sample actions $Y_1,\dots,Y_n \iidsim \pi_\theta(\cdot \mid X)$.
    \FOR{each action $Y_i$ and reward function $R_k$}
        \STATE Calculate the effective advantage: 
        \vspace{-3pt}
        \begin{equation*}
            \hspace{-1cm} A_i^{(k)} = \max_{1\leq j \leq n}\!R_k(X, Y_j) - \!\!\max_{\substack{1\leq j \leq n \\ j \neq i}}\!\! \, R_k(X, Y_j)
        \end{equation*} \vspace{-11pt}
    \ENDFOR
    \STATE Compute gradient estimator:
    \vspace{-3pt}
    \begin{equation*}
        \hat{g} = \sum_{i=1}^n \Big( 
            \sum_{k=1}^m \alpha_k \cdot A_i^{(k)} 
        \Big)\nabla_\theta \log \pi_\theta(Y_i \mid X)
    \end{equation*}
    \vspace{-3pt}
    \STATE Update weights $\theta' \leftarrow \texttt{Optimizer}(\theta, \hat{g})$
    \end{algorithmic}
    \vspace{-6pt}
    \rule{\linewidth}{0.8pt}
    \vspace{-30pt}
\end{minipage}
\vspace{-15pt}
\end{wrapfigure}

Proposition~\ref{prop:non-uniform} allows us to \textit{predict}, given reward function probabilities $(\alpha_1, ..., \alpha_m)$, how frequently the optimal policy under \methodnameshort+Max will sample from \textit{each} of the distinct correct answers (corresponding to each reward function). This gives us the ability to precisely predict and/or control the diversity in the optimal policy through the reward function distribution. Figure~\ref{fig:rosa-max-nonuniform-simplex} numerically illustrates this point: the \methodnameshort+Max optimum smoothly interpolates between uniformly sampling optimal actions, to preferentially sampling the optimal action under the more probable reward function.

\section{Experiments}
\label{sec:experiments}

We now evaluate empirically how our proposed objective function, \textbf{\methodnameshort+Max}, and its variant \textbf{\methodnameshort+Softmax}  compare with the standard policy gradient objective and other variants in experiments meant to illustrate a variety of real-world applications.

A practical \methodnameshort+Max gradient update is provided in Alg.~\ref{alg:rosa-loo}. Although we write it iteratively for clarity, it can be implemented efficiently via matrix operations (over $n$ actions and $m$ sampled reward functions). Note that the same algorithm can be applied to any (multi)set functions by replacing the $\max$ with the desired set function.

\begin{figure}[t]%
    \centering 
    \begin{subfigure}[b]{0.47\linewidth}
        \centering
        \includegraphics[width=\textwidth]{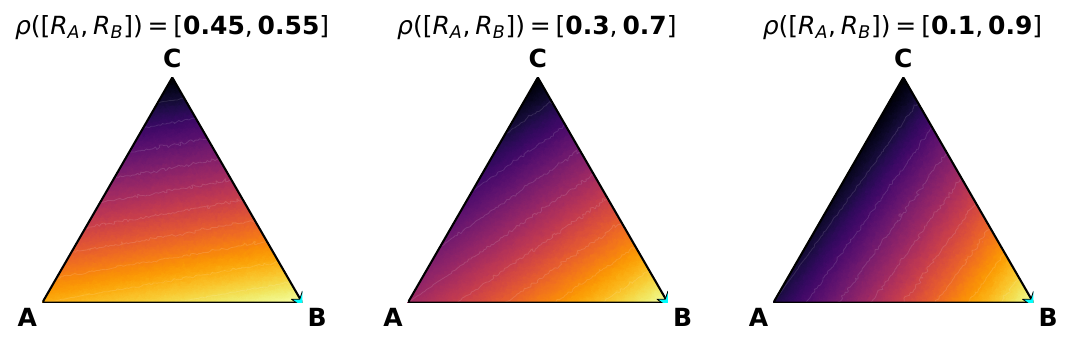}
        \caption{Vanilla PG} 
    \end{subfigure}
    \hfill
    \begin{subfigure}[b]{0.47\linewidth}
        \centering
        \includegraphics[width=\textwidth]{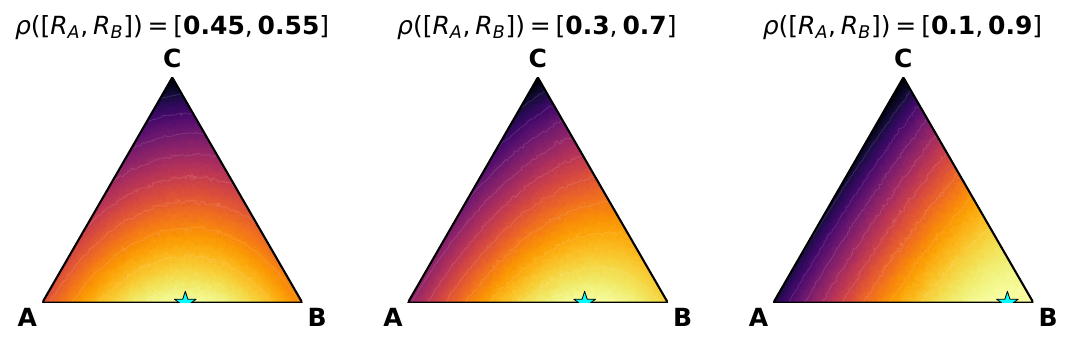}
        \caption{\methodnameshort+Max}
    \end{subfigure}
    \caption{Global objective landscape over simplex of 3-category policies, with non-uniformly weighted reward functions $R_A$ and $R_B$. Vanilla PG can only induce global optima at deterministic policies, while \methodnameshort+Max naturally induces global optima at optimal stochastic policies. The probability of sampling each unique optimal action is characterized in Proposition~\ref{prop:non-uniform}.}
    \label{fig:rosa-max-nonuniform-simplex}
    \sneakynegvspace
    \vspace{-0.17in}
\end{figure}

\subsection{Diverse, competing preferences}

We first investigate the regime of \textbf{conflicting preferences}, where $R \sim \rho$ represents a distribution over diverse, contradictory reward functions. In this setting, vanilla policy gradient is fundamentally incapable of optimization: by aggregating the return as $\bar{r}(Y) = \mathbb{E}_{R \sim \rho}[R(Y)]$, opposing reward signals mutually cancel out, resulting in no reward. 

We construct a didactic experiment with a small transformer outputting only integers.
There are two competing reward functions: $R_1$ is $+1$ if all generated integers have the same parity (i.e. same odd/even-ness) as the last integer in the prompt, $-1$ if all generated integers have the opposite parity as the last integer in the prompt, and $0$ if the generated integers have mixed parities. $R_2$ is the negative of $R_1$ (see Figure~\ref{fig:parity-task-schematic} for example and explanation). As Figure~\ref{fig:opposing-penalty-2RFns} demonstrates, \methodnameshort is the only method capable of making any progress at all in this setting.

{
}

\begin{figure}[t]
    \centering
    \begin{subfigure}[b]{0.4\linewidth}
        \centering
        \resizebox{\linewidth}{!}{
        \input{fig/toy/sec4-2_structure-opposite-reward-toy}
        }
        \caption{Sequence generation task with competing reward functions. $R_1$ prefers answers with the same parity as the last question token, $R_2$ prefers opposite parity.}
        \label{fig:parity-task-schematic}
    \end{subfigure}
    \hfill
    \begin{subfigure}[b]{0.58\linewidth}
        \centering
        \includegraphics[width=\linewidth]{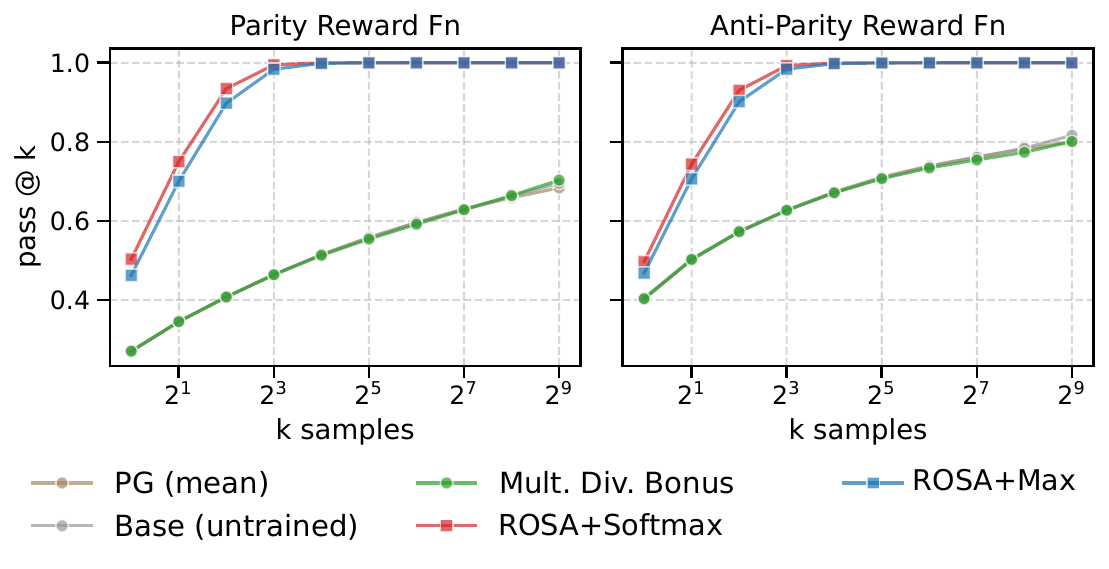}
        \caption{Evaluation result of different gradient estimators.
        }
    \end{subfigure}
    \caption{Comparison of vanilla PG, \textbf{\methodnameshort+Max} and \textbf{\methodnameshort+Softmax} in scenario of competing preferences. PG (even if augmented with a multiplicative diversity reward, \texttt{Mult.\ Div.\ Bonus}) suffers from reward cancellation, resulting in \textit{no learning}---its pass@$k$ curve is identical to the \texttt{Base} (untrained) model. However, \methodnameshort+Max and \methodnameshort+Softmax cleanly balance both preferences: starting at $\sim$50\% accuracy and generating both answer types as $k$ increases. Details in Appendix~\ref{app:exp-competing-pref-small-transformer}.}
    \label{fig:opposing-penalty-2RFns}
    \sneakynegvspace
    \vspace{-0.22in}
\end{figure}

\subsection{Multiple correct answers}
\label{sec:multiple-correct-answers}

Many tasks in the real world can naturally contain multiple correct answers (or correct actions). For instance, a coding task asking the user to ``reverse a Python list'' can be implemented in a number of syntactically distinct, yet functionally equivalent ways, or a task of discovering a drug having a high reward (defined by some properties) may have multiple drug molecules achieving the highest reward. In these scenarios, the user is not interested in finding a deterministic policy which achieves the highest reward on one solution but finding a stochastic policy which discovers many correct solutions. 

We can model this scenario by defining a reward function distribution over multiple reward functions, each corresponding to a unique, correct answer. Note that optimizing the sum/average i.e. ``joint'' reward function $\bar{R}(Y_i) = \sum_{\ans \in \ansset} R_a(Y_i)$, reduces to a regular reward function that gives ``1'' to all correct answers $Y_i \in \ansset$, and 0 otherwise.
We construct a setting in mathematical reasoning where we prefer answers that are \textit{correct}, but we want diversity in the answer lengths: short answers are concise and to the point, while long answers can be more didactic and detailed. Specifically, we train with four different reward functions with different length preferences. Figure~\ref{fig:multi-correct-math-lengths} shows the output pass@$k$ from sampling the resulting policies, where \methodnameshort achieves the same overall correctness as regular policy gradient training, but is able to sample from all length preferences specified. We show short and long example responses from a \methodnameshort trained policy in Figure~\ref{fig:math-multi-len-generations}.  

\begin{figure}[htb]
    \centering
    \begin{subfigure}[b]{0.98\linewidth}
        \centering
        \includegraphics[width=\linewidth]{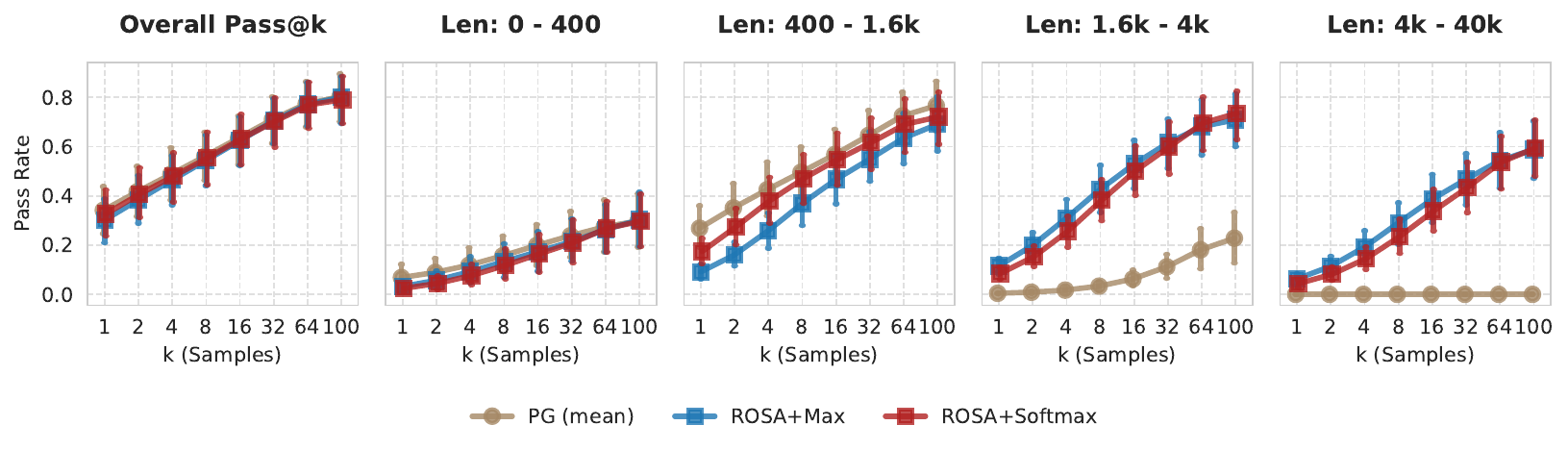}
    \end{subfigure}
    
    \caption{Comparison of \methodnameshort+Max/Softmax to PG in a setting where we train on \texttt{MATH} and with multiple reward functions preferring different lengths. Each reward function gives 1 if the answer is both correct \textit{and} within its preferred character length. Training with vanilla PG results in generations mostly in the 400-1.6k length range, while the \methodnameshort-trained policy samples from all four preferred length ranges, without sacrificing any accuracy (overall pass@$k$ rate). %
    }
    \sneakynegvspace
        \vspace{-0.15in}

    \label{fig:multi-correct-math-lengths}
\end{figure}

\subsection{Learning with unknown ground-truth rewards}

\begin{figure}[htb!]
    \centering
    \begin{subfigure}[b]{0.98\linewidth}
        \centering
        \includegraphics[width=\linewidth]{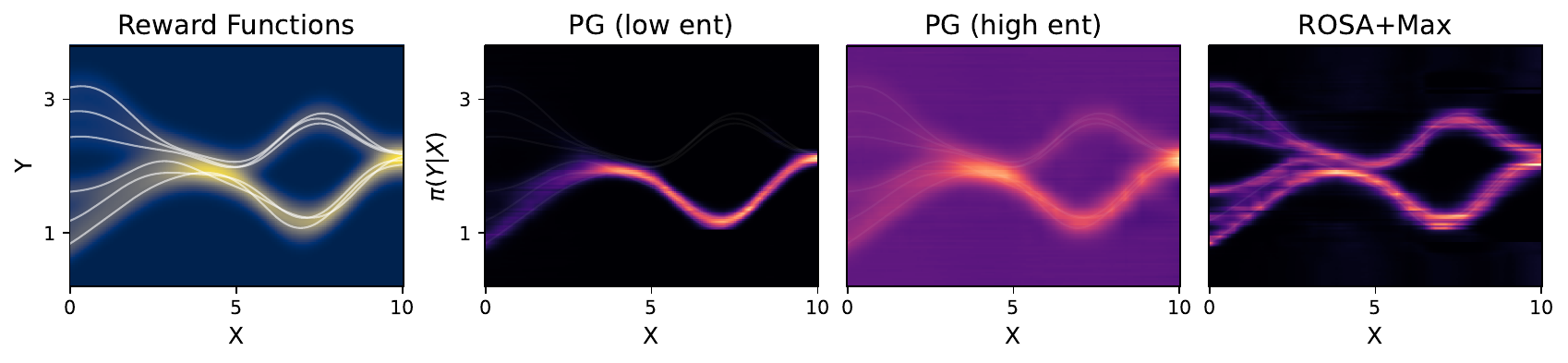}
    \end{subfigure}
    \caption{
    Left: Reward functions; white lines denote maximum reward $y$ for each $x$, with rewards decaying exponentially away from the maximum. Right: Trained policy $\pi(Y|X)$.}
    \label{fig:uncertain-rewards-2d}
    \vspace{-0.15in}
\end{figure}
We now consider the case where we are \textit{uncertain} about the \textit{true} reward function. This naturally occurs in reward modelling, or if LLMs are used as the reward function. Here, due to finite data or stochasticity in the judge, we only have an approximately correct reward function. This epistemic uncertainty is naturally expressed as a reward function distribution $R \sim \rho$, and we will subsequently show that \methodnameshort+Max/Softmax naturally captures the uncertainty in the reward function in the resulting policy's action distribution.
We first set up a didactic task where there are multiple reward functions that do not always agree. Specifically, we consider the case of one-dimensional state $X$ and action $Y$; see Figure~\ref{fig:uncertain-rewards-2d}(left), illustrating multiple diverse, non-equally weighted reward functions.
Figure~\ref{fig:uncertain-rewards-2d}(right) displays the results of training \methodnameshort+Max as well as PG with low/high entropy regularization. PG training with low-entropy regularization converges to the highest weighted reward functions, while high-entropy regularization results in wider coverage but also action probabilities over many bad actions that are not optimal under \textit{any} reward functions. \methodnameshort+Max covers all reward modes in proportion to their reward function's weight. See Appendix~\ref{app:reward-function-uncertainty-toy} for more experimental details and additional figures comparing specific slices of $\pi(Y|X)$ between methods. 

\begin{figure}[htb!]
    \centering
    \begin{subfigure}[b]{0.55\linewidth}
        \centering
        \includegraphics[width=\linewidth]{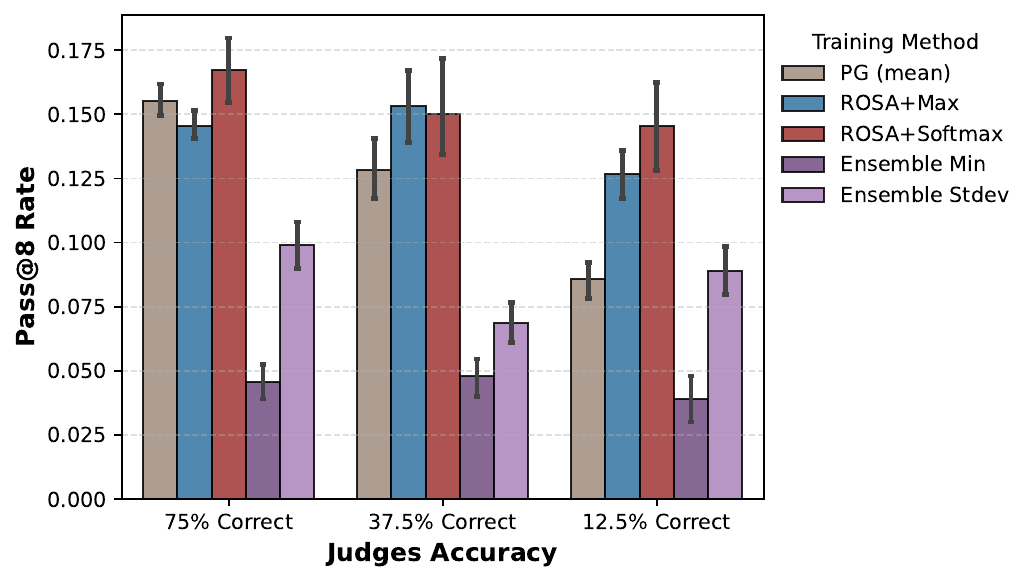}
        \caption{In domain (\texttt{MATH}) evaluation with different judge mixtures}
        \label{fig:math-noisy-judge}
    \end{subfigure}
    \hfill %
    \begin{subfigure}[b]{0.42\linewidth}
        \centering
        \includegraphics[width=\linewidth]{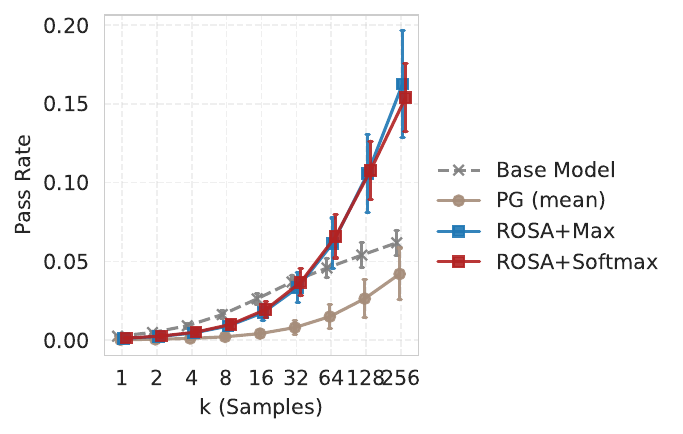}
        \caption{Out-of-distribution pass@$k$ performance on unseen \texttt{AIME 2025} problem set}
        \label{fig:aime25-ood}
    \end{subfigure}
    \caption{Gemma 2B performance when trained on MATH where the reward is given by an ensemble of noisy judges. (\ref{fig:math-noisy-judge}) In distribution evaluation performance. (\ref{fig:aime25-ood}) Out-of-distribution pass@$k$ evaluation on unseen AIME 2025 set. Ensemble min and stddev refer to taking $\min{((r_i)_{i=1}^m)}$ or $\text{avg}((r_i)_{i=1}^m) - \text{stddev}((r_i)_{i=1}^m)$ of the $m$ rewards. All details are in Appendix~\ref{app:math-noisy-judge}}
    \sneakynegvspace
\end{figure}

Finally, we extend our evaluation to an ``ensemble of judges'' setting for LLM post-training. We construct a mathematical reasoning task in which each problem is paired with a panel of eight judges. A response receives reward of 1 from a judge if it matches the judge's target, though only some judges are correct, and the learner is not told which ones. We train under varying levels of judge noise, with 6, 3, or 1 correct judges out of 8 (corresponding to 75\%, 37.5\%, and 12.5\% judges' accuracy in Figure~\ref{fig:math-noisy-judge}). Across these settings, \methodnameshort is more robust to noisy rewards than both standard PG and ensemble-regularized pessimistic baselines (Figure~\ref{fig:math-noisy-judge}). To test out-of-distribution generalization, we further evaluate models trained with noisy MATH judges on AIME2025. \methodnameshort-trained models achieve higher out-of-distribution pass@$k$ (Figure~\ref{fig:aime25-ood}).

\section{Related work}
\label{sec:related-work}
\methodnameshort generalizes several existing RL objectives. Specifically, in the \textit{single reward function} case, \methodnameshort with the set mean function $f_{\text{mean}} \big((U_i)_{i=1}^n\big) = \frac{1}{n}\sum_{i=1}^n U_i$ recovers the vanilla PG criterion (Eq.~\eqref{eq:objective}), while the set max $f_{\text{max}}\big((U_i)_{i=1}^n\big) = \max_{1\leq i \leq n} U_i$ coincides with the ``best-of-$n$'' training objective \citep{stiennon2020learning}.
With multiple reward functions, the set mean function $f_{\text{mean}}$ recovers linear scalarization from multi-objective RL \citep{vamplew2011empirical}.

\textbf{Improving diversity in post-training.} Many recent works \citep{huang2024self, padmakumar2023does,gx2025kl} have pointed to the problem of mode collapse in RL post-trained models compared to supervised fine-tuned models. \citet{li2025jointly} proposes to address this by jointly training the model on a reward function which is a product of diversity term  and task reward. The diversity term is computed as a function of pairwise distances over a batch of samples for each prompt where the more unique samples are assigned a higher reward. While this approach preserves the diversity during post-training and shows significant gains in performance, we observe that this diversity bonus undesirably distorts the ordering of sub-optimal policies (Figure~\ref{fig:mult-diversity-uniform-simplex}). Similarly, \citet{hamid2025polychromic,orney2026poly} train on product of reward and diversity defined at a set level, at the cost of having a non reward-maximizing policy as the optimum of this objective (Figure~\ref{fig:simplex-set-multiplicative-divs}). We discuss in detail in Appendix~\ref{app:related-work-diversity-bonus}.

\textbf{Optimization over a set of responses.} The best-of-$n$ approach to inference-time policy improvement is a popular approach to improving LLM performance \citep{stiennon2020learning}. Several recent works have studied the modification of RL objectives to take this inference-time technique into account \citep{koyamada2024emergence,beirami2024theoretical,balashankar2024infalign,chow2024inference,tang2025optimizing,walder2026pass,chen2025post}, and other works use the implicit resulting improved policy in training objectives \citep{gui2024bonbon,sessa2024bond}.
Distinct from our contributions of dealing with distribution of reward functions, these methods are concerned with optimization of a single reward function, and thus have deterministic policies as objective optima. Earlier related work in this vein is motivated by recommender systems \citep{radlinski2008learning}, including online learning methods \citep{kale2010non,yue2011linear,dimakopoulou2019marginal} and deep reinforcement learning \citep{sunehag2015deep}. \citet{fard2011non} study the problem of identifying all near-optimal actions in MDPs via mixed-integer programming methods. 

Another closely related approach is the ReMax objective \citep{koyamada2024emergence,nishimori2026emergence}, which coincides with ROSA+Max (Equation~\eqref{eq:rosa-max-criterion}) in the bandit setting. Like ROSA+Max, ReMax observes that the optimal bandit policy under this objective is stochastic, unlike in standard policy-gradient methods. However, while the initial motivations are similar, ReMax primarily targets improved exploration in MDPs. Algorithmically, \citet{nishimori2026emergence} propose an expected-improvement-based estimator for best-of-$n$ gradients, but dispense with explicit reward or value distributions, instead interpreting the non-stationary Q-function being learned as samples from the action-value distribution. In contrast, \methodnameshort is designed specifically for generative model post-training: it uses different gradient estimators and, crucially, maintains explicit reward uncertainty as a natural consequence of real-world, multiply satisfiable task constraints.  Whereas ReMax uses stochasticity instrumentally to discover better deterministic MDP policies, \methodnameshort treats policy stochasticity and reward uncertainty as central objects of the optimization problem.

\textbf{Multi-objective reinforcement learning.} Multi-objective reinforcement learning \citep{roijers2013survey,hayes2022practical} focuses on modelling problems that naturally have several, potentially conflicting, notions of reward. Key questions include optimization of reward subject to cost constraints \citep{altman2021constrained}, and discovery of policies on the Pareto frontier \citep{gabor1998multi,mannor2001steering,abdolmaleki2020distributional,bahlous2026vector}. Most closely connected to the problem we study is \emph{scalarization}, in which a single scalar objective is constructed from several reward signals. However, these approaches are typically limited to expected utilities, which we show in Appendix~\ref{app:morl-scalarization} are either insufficiently expressive to capture our desired diverse policies as solutions, or otherwise suffer from bias issues, unlike the ROSA objective proposed here.
Concurrent with this paper, \citet{bahlous2026vector} propose a method for Pareto optimization that combines multiple reward functions and sets of actions; we discuss comparison with this work in more detail in Appendix~\ref{app:morl-pareto}.

\section{Conclusion}
\vspace{-0.1in}
We presented a novel, generalized reinforcement learning framework that unifies optimization for \textit{sets of actions} over \textit{distributions of reward functions}. By shifting the objective away from maximizing expected scalar returns, we obtain policy gradient objectives whose reward-maximizing solutions are by design \textit{distributions}, thereby resolving the tension between reward maximization and policy entropy. This addresses a variety of limitations with standard policy gradients,
such as mode collapse and brittleness under reward function errors. We provide rigorous theoretical analysis for the behaviour, optimization efficiency, and optimal solution of this objective, and also provide empirical evidence at small and large scales for how \methodnameshort can be applied to a wide variety of example settings. Ultimately, this work provides a principled foundation for training stochastic policies that are as diverse and adaptable as the environments they inhabit.

\subsection*{Acknowledgement}

The authors are grateful to Tom Schaul and Guillaume Desjardins for their thoughtful comments on a draft of the paper, and for insightful discussions with Yinlam Chow, Arian Hosseini, Tom Zahavy, Kalesha Bullard, and Andreas Kirsch. We would further like to thank the Agency team and the Montreal team at large for various inspirational conversations.

%% file: fig/toy/sec4-2_structure-opposite-reward-toy.tex
\begin{tikzpicture}[
    font=\footnotesize,
    token/.style={
        draw,
        rounded corners=1.5pt,
        inner sep=1.5pt,
        minimum width=1.45em,
        minimum height=1.45em
    },
    odd/.style={token},
    even/.style={token, dashed},
    leftlabel/.style={anchor=west},
    reward/.style={anchor=west}
]

\definecolor{rewardgreen}{RGB}{34,94,34}
\definecolor{rewardred}{RGB}{120,35,35}
\definecolor{rewardgrey}{RGB}{80,80,80}

\def\xlabel{0.0}

\def\xQtokA{1.2}
\def\xQtokB{1.8}
\def\xQtokC{2.4}

\def\xAtokA{2.2}
\def\xAtokB{2.8}
\def\xAtokC{3.4}

\def\xnote{3.2}
\def\xreward{3.8}

\def\yQ{2.7}
\def\yP{1.8}
\def\yA{0.9}
\def\yM{0.0}

\node[leftlabel] at (\xlabel,\yQ) {Q:};
\node[odd]  at (\xQtokA,\yQ) {1};
\node[odd]  at (\xQtokB,\yQ) {3};
\node[odd, very thick]  at (\xQtokC,\yQ) {3};
\node[anchor=west] at (\xnote,\yQ) {\textbf{Final question token odd}};

\node[leftlabel] at (\xlabel,\yP) {Parity Ans};
\node[odd]  at (\xAtokA,\yP) {7};
\node[odd]  at (\xAtokB,\yP) {7};
\node[odd]  at (\xAtokC,\yP) {7};
\node[reward] at (\xreward,\yP) {$\textcolor{rewardgreen}{R_1=+1},\;\textcolor{rewardred}{R_2=-1}$};

\node[leftlabel] at (\xlabel,\yA) {Anti-parity};
\node[even] at (\xAtokA,\yA) {6};
\node[even] at (\xAtokB,\yA) {6};
\node[even] at (\xAtokC,\yA) {6};
\node[reward] at (\xreward,\yA) {$\textcolor{rewardred}{R_1=-1},\;\textcolor{rewardgreen}{R_2=+1}$};

\node[leftlabel] at (\xlabel,\yM) {Mixed Ans};
\node[odd]  at (\xAtokA,\yM) {5};
\node[even] at (\xAtokB,\yM) {6};
\node[odd]  at (\xAtokC,\yM) {7};
\node[reward] at (\xreward,\yM) {$\textcolor{rewardgrey}{R_1=0},\;\textcolor{rewardgrey}{R_2=0}$};

\path (current bounding box.south) ++(0,-0.1);

\end{tikzpicture}

%% file: content/03_appendix.tex
\section*{\centering APPENDICES}

\section{Proofs}\label{sec:proofs}

In this section, we collect statements and proofs of results from the main paper.

\subsection{\methodnameshort policy gradient}
\label{app:rosa-policy-gradient}

\begin{restatable}{proposition}{propROSAPolicyGrad}\label{prop:rosa-general-grad}
    
    Let $\setY = (Y_i)_{i=1}^n$ denote a \textit{multiset} of actions induced by $Y_1, \dots, Y_n$. We write $\setY \sim \pi_\theta$ to mean the $n$ actions are sampled i.i.d.\ from policy $\pi_\theta$. Denote by $\rho$ a distribution over reward functions $R: \mathcal{X} \times \mathcal{Y} \rightarrow \mathbb{R}$, and let $R \sim \rho$ be a draw from this distribution. For brevity, we  write $R(X, \setY) = (R(X, Y_i))_{i=1}^n$ as the reward multiset. Let $f$ be a function defined over multisets of actions. The generic ROSA objective is
    \begin{equation}
        \mathcal{J}_{\mathrm{\methodnameshort}}(\pi_\theta) = \mathbb{E}_{X \sim \mu \,,\,  \setY \iidsim \pi_\theta (\cdot|X)}\bigg[ \mathbb{E}_{R \sim \rho} \Big[
            f(R(X, \setY))
        \Big] \bigg] \,,
    \end{equation}
    and the gradient of the objective at a sampled state $X \sim \mu$ is 
    \begin{equation}
        \mathbb{E}_{\setY \iidsim \pi_\theta (\cdot|X)} \bigg[ \mathbb{E}_{R \sim \rho} \Big[
            \sum_{i=1}^n \Big( 
                f\big(R(X, \setY)\big) - b\big(R(X, \setY_{-i})\big)
            \Big) \nabla_\theta \log \pi_\theta (Y_i|X)
        \Big] 
        \bigg] \,,
    \end{equation}
    where $b$ is an arbitrary multiset function that forms an optional control variate, and $\setY_{-i}$ indicates the multiset $\setY = (Y_j)_{j=1}^n$ with element $Y_i$ removed.
\end{restatable}
\begin{proof}
    The proof is a straightforward extension of analysis by \citet{tang2025optimizing} in the single-reward setting, to deal with randomization over reward functions. 
    The core technique (in the single-reward setting) is to treat the vector $(Y_1,\ldots,Y_n)$ as a single ``macro action'', with reward $f(R(X, \setY))$. Generally, for any multiset function $g$, we have per the usual score function trick \citep{williams1992simple},
    \begin{align*}
        \nabla_\theta \, \mathbb{E}_{\setY \iidsim \pi_\theta (\cdot|X)} \big[
            g(X, \setY) 
        \big] = \mathbb{E}_{\setY \iidsim \pi_\theta (\cdot|X)} \big[
            g(X, \setY) \, \nabla_\theta \log \pi_\theta (\setY | X)
        \big] \,,
    \end{align*}
    as an expression for the gradient. The statement then follows by noting that since the actions comprising $\setY$ are independent, $\log \pi_\theta(\setY|X) = \sum_{i=1}^n \log \pi_\theta(Y_i |X)$, which gives us,
    \begin{align*}
        \mathbb{E}_{\setY \iidsim \pi_\theta (\cdot|X)} \bigg[
            g(X, \setY) \sum_{i=1}^n  \nabla_\theta  \log \pi_\theta (Y_i | X)
        \bigg] \,.
    \end{align*}
    Extending this to $g(X, \setY) = \mathbb{E}_{R \sim \rho}[f(R(X, \setY))]$ yields,
    \begin{align*}
        \nabla_\theta \mathcal{J}_{\text{ROSA}} (\pi_\theta) = \mathbb{E}_{\setY \iidsim \pi_\theta (\cdot|X)} \bigg[ \mathbb{E}_{R \sim \rho} 
            \Big[
                f\big(R(X, \setY)\big) \nabla_\theta \log \pi_\theta(\setY | X)
            \Big]
        \bigg]\,.
    \end{align*}
    Finally, by independence, any random variable independent of $Y_i$ can stand as a control variate for $\nabla_\theta \log \pi_\theta (Y_i | X)$ per the score function property $\mathbb{E}_{Y_i \sim \pi_\theta} [\nabla_\theta \log \pi_\theta (Y_i|X)] = 0$. It follows that,
    \begin{align*}
        \mathbb{E}_{\setY \iidsim \pi_\theta (\cdot|X)} \Big[
            b(X, \setY_{-i}) \, \nabla_\theta  \log \pi_\theta (Y_i | X)
        \Big] = 0 \,.
    \end{align*}
    Subtracting this expectation-zero quantity from the per-sample score gives us the gradient as expressed in the statement of the result.
\end{proof}

Note $b$ can be \textit{any} multiset function, including a function which only make use of a subset of $\setY_{-i}$. We opt to use $b = f$ in the present work for simplicity.

\subsection{Uniform rewards, \methodnameshort+Max optimal policy}
\label{app:proof-rosa-max-uniform}

\propOptimalBinaryDistinct*

\begin{proof}
    The objective value for a policy putting probability $p_i$ on $y^*_i$ is
    \begin{align*}
        \frac{1}{m}\sum_{i=1}^m \bigg( 1 - (1-p_i)^n \bigg) \, .
    \end{align*}
    The constraint is that we have $\sum_{i=1}^m p_i \in [0,1]$. We will not directly impose constraints of non-negativity on the $(p_i)_{i=1}^m$, and will instead verify that the solution we obtain below automatically satisfies these conditions. Clearly, it is optimal to take $\sum_{i=1}^m p_i = 1$. We can then consider the Lagrangian associated with this optimization problem to understand which policies are optimal. The Lagrangian, with Lagrange multiplier $\lambda$, is
    \begin{align*}
        \frac{1}{m} \sum_{i=1}^m \bigg( 1 - (1-p_i)^n \bigg) - \lambda \bigg(\sum_{i=1}^m p_i - 1\bigg) \, .
    \end{align*}
    The derivative of the Lagrangian with respect to $p_i$ is
    \begin{align*}
        -\frac{n}{m} (1-p_i)^{n-1} - \lambda \, .
    \end{align*}
    Equating these to 0 for all $p_i$ means that the $p_i$ must be equal, and hence the optimizer to this linearly-constrained concave maximization problem is when the $p_i$ are all equal, as required. %
\end{proof}

\lemHessian*

\begin{proof}
    As in the proof of Proposition~\ref{prop:optimal-binary-distinct}, note that the objective value for a policy putting probability $p_i$ on $y^*_i$ is
    \begin{align*}
        \frac{1}{m}\sum_{i=1}^m \bigg( 1 - (1-p_i)^n \bigg) \, .
    \end{align*}
    Also by Proposition~\ref{prop:optimal-binary-distinct}, the optimal policy has $p_i = 1/m$ for $i=1,\ldots,m$. We can then immediately calculate the entries of the Hessian as required.
\end{proof}

\subsection{Non-uniform rewards, \methodnameshort+Max optimal policy}
\label{app:proof-rosa-max-non-uniform}

\propNonUniform*

\begin{proof}
    Following a similar procedure as Proposition~\ref{prop:optimal-binary-distinct}, we have the objective $\sum_{i=1}^m \alpha_i \, (1 - (1-p_i)^n) $ where $p_i$ is the probability the policy samples the correct action under the $i$-th reward function. The Lagrangian, including the nonnegative constraint, is,
    \begin{align*}
        \sum_{i=1}^m \alpha_i \, \Big( 1 - (1-p_i)^n \Big) - \lambda \Big( \sum_{i=1}^m p_i - 1 \Big) + \sum_{i=1}^m \mu_i p_i \,.
    \end{align*}
    The partial derivatives with respect to each $p_i$ have the form $\alpha_i n \big(1 - p_i)^{n-1} - \lambda + \mu_i \,$.
    Setting the derivatives to zero, we get the generic relationship describing the constrained optimum,
    \begin{equation}
        n \alpha_i \, (1 - p^*_i)^{n-1} = \lambda - \mu^*_i \,,
        \label{eq:non-uniform-generic-relationship}
    \end{equation}
    with conditions $p^*_i \geq 0$, $\mu^*_i \geq 0$, $\mu^*_i p^*_i=0$, and $\sum_i p^*_i = 1$. 

    If $p^*_i > 0$, then $\mu^*_i = 0$ (complementary slackness), and the solution is (denoting as $A_i$ for brevity),
    \begin{align}
        p^*_i = A_i \vcentcolon= 1 - \Big( \frac{\lambda}{n \alpha_i} \Big)^{\frac{1}{n-1}} \,.
        \label{eq:non-uniform-active-set-probs-lambda}
    \end{align}
    Note $A_i > 0$ also implies $p^*_i > 0$ (and therefore $p^*_i = A_i$). To see this, note $A_i > 0$ implies $n \alpha_i > \lambda$. If $p^*_i = 0$, then $n \alpha_i = \lambda - \mu^*_i \leq \lambda$ (per Equation~\eqref{eq:non-uniform-generic-relationship} and $\mu^*_i \geq 0$) which is a contradiction. Since $p^*_i > 0$ if and only if $A_i > 0$, this also means $p^*_i=0$ when $A_i \leq 0$. Taken together, we have,
    \begin{equation*}
        p^*_i = \max \bigg(0, \, 1 - \Big( \frac{\lambda}{n \alpha_i} \Big)^{\frac{1}{n-1}} \bigg) \,,
    \end{equation*}
    with a unique value of $\lambda$ such that $\sum_i p^*_i = 1$. Since this function is strictly increasing in $\alpha_i$, we know that the top-$k$ largest $\alpha$'s will have non-zero $p^*_i$, for some value of $k$ (including $k=m$).
    Consider only this active set $\{p^*_1 , \dots , p^*_k\}$, $p^*_i > 0$. Equation~\eqref{eq:non-uniform-active-set-probs-lambda} and $\sum_i p^*_i = 1$ gives,
    \begin{equation*}
        \lambda = n \bigg(\frac{k-1}{\sum_{i=1}^k \alpha_i^{-1 / (n-1)}}\bigg)^{n-1} \,.
    \end{equation*}
    Substituting this back into Equation~\eqref{eq:non-uniform-active-set-probs-lambda} gives us the probabilities for the active set,
    \begin{equation*}
        p_i^* = 1 - \frac{(k-1) \, \alpha_i^{-1/(n-1)}}{\sum_{j=1}^k \alpha_j^{-1/(n-1)}} \, ,
    \end{equation*}
    as required.
\end{proof}

\subsection{Uniform rewards, \methodnameshort+General set function optimal policy}
\label{app:proof-unif-rewards-general-f}

Before proving Theorem~\ref{thm:optimal-binary-distinct}, we state and prove an auxiliary result that will be helpful in our eventual argument. We remind the reader we consider set functions $f(R(X, Y_1),\ldots,R(X, Y_n))$, and their corresponding success-count functions $\tilde{f}(\sum_{i=1}^n R(X, Y_i))$.

\begin{lemma}\label{lem:unif-reward-general-f}
     Assume we have $m$ reward functions $(r_i)_{i=1}^m$, with each $r_i$ one-hot for a distinct action $y^*_i$, and probability $\rho(r_i)$ under the distribution $\rho$. The overall objective can be written as,
    \begin{equation}
        \mathbb{E}_{\setY \sim \pi} \Big[
            \mathbb{E}_{R \sim \rho} \big[f(R(\setY)) \big] 
        \Big] = \mathbb{E}_{\setY \sim \pi} \Big[
            \sum_{i=1}^m \rho(r_i) \tilde{f}(S_i) \big] 
        \Big] \,,
    \end{equation}
    where $S_i = \sum_{j=1}^n r_i(Y_j)$. Further, define $p_i = \mathbb{E}_{Y\sim\pi}[r_i(Y)]$ to be the probability of sampling an action obtaining a reward of 1 under $r_i$. The partial derivative of the overall objective w.r.t.\ $p_i$ is,
    \begin{equation}
        \frac{\partial}{\partial p_i}\mathbb{E}_{\setY \sim \pi} \Big[\sum_{i=1}^m \rho(r_i) \tilde{f}(S_i) \Big] 
        = \rho(r_i) \, n\, \mathbb{E}_{S_i'}\Big[
            \tilde{f}(S_i'+1)  - \tilde{f}(S_i')
        \Big] \,,
    \end{equation}
    where $S_i' \sim \mathrm{Binom}(n-1, p_i)$.
\end{lemma}
\begin{proof}[Proof of Lemma~\ref{lem:unif-reward-general-f}]
    The first statement follows by writing the expectation over reward functions out explicitly, and using the definition of the success-count function $\tilde{f}$ from $f$ in the binary reward setting introduced in the main paper.
    Next, since $r_i$ is binary, $r_i(Y_j)$ is a Bernoulli random variable, $r_i(Y) \sim \text{Bernoulli}(p_i)$. It follows that $S_i \sim \text{Binom}(n, p_i)$, so that 
    \begin{equation}
        \mathbb{P}(S_i = k) = {n \choose k} p_i^k (1-p_i)^{n-k} \,.
    \end{equation}
    We can write the derivative of the binomial; with some algebra we get,
    \begin{align*}
        \frac{\partial}{\partial p_i} \mathbb{P}(S_i = k) &= {n \choose k} \bigg[
            p_i^{k-1} (1-p_i)^{n-k-1} \Big[
                k(1-p_i) - (n-k) p_i
            \Big]
        \bigg] \,,\\
        &= k {n\choose k} p_i^{k-1} (1-p_i)^{n-k} - (n-k){n \choose k}p_i^k (1-p_i)^{(n-1)-k} \,, \\
        &= n \Big[
            \mathbb{P}(S_i' = k-1) - \mathbb{P}(S_i'=k)
        \Big] \,,
    \end{align*}
    where $S_i' \sim \text{Binom}(n-1, p_i)$. Further, we can write,
    \begin{align*}
        \frac{\partial}{\partial p_i} \mathbb{E}_{S_i}\Big[ f(S_i, n) \Big] &= \frac{\partial}{\partial p_i} \sum_{k=0}^n \mathbb{P}(S_i=k) \, \tilde{f}(k) \,, \\
        &= \sum_{k=0}^n \, \tilde{f}(k)  \, \frac{\partial}{\partial p_i} \mathbb{P}(S_i=k) \,, \\
        &=  \sum_{k=0}^n \, \tilde{f}(k)  \, n \big[ \mathbb{P}(S_i' = k-1) - \mathbb{P}(S_i'=k) \big] \,, \\
        &= n \Big[
            \sum_{k=0}^n \tilde{f}(k) \mathbb{P}(S_i' = k-1) - \sum_{k=0}^n \tilde{f}(k) \mathbb{P}(S_i' = k)
        \Big] \,, \\
        &= n \Big[
            \sum_{h=0}^{n-1} \tilde{f}(h+1) \mathbb{P}(S_i' = h) - \sum_{k=0}^{n-1} \tilde{f}(k) \mathbb{P}(S_i' = k)
        \Big] \,, \\
        &= n \sum_{k=0}^{n-1} \mathbb{P}(S_i' = k) \Big[
            \tilde{f}(k+1)  - \tilde{f}(k) 
        \Big] \,, \\
        &= n \, \mathbb{E}_{S_i'}\Big[\tilde{f}(S_i'+1)  - \tilde{f}(S_i') \Big] \,.
    \end{align*}
    Putting things together, we have,
    \begin{equation}
        \frac{\partial}{\partial p_i}\mathbb{E}_{\setY \sim \pi} \Big[\sum_{j=1}^m \rho(r_j) \tilde{f}(S_j) \Big] 
        = \rho(r_i) \cdot n \, \mathbb{E}_{S_i'}\Big[
            \tilde{f}(S_i'+1)  - \tilde{f}(S_i')
        \Big] \, ,
    \end{equation}
    as required.
\end{proof}

We are now ready to prove Theorem~\ref{thm:optimal-binary-distinct}.

\thmOptimalBinaryDistinct*

\begin{proof}
    Under the assumptions of the statement, Lemma~\ref{lem:unif-reward-general-f} shows that the partial derivative of the objective with respect to $p_i$ is proportional to 
    \begin{equation}
    \mathbb{E}_{S_i'}\Big[
            \tilde{f}(S_i'+1)  - \tilde{f}(S_i')
        \Big] \, ,
    \end{equation}
    where $S_i' \sim \text{Binom}(n-1, p_i)$. 
    If $\tilde{f}$ is strictly increasing and concave, then the function $h(k) = \tilde{f}(k+1) - \tilde{f}(k)$ is strictly decreasing and positive. It then follows that if $S' \sim \text{Binom}(n-1, p_i)$, we have that
    \begin{align*}
         \mathbb{E}_{S_i'}\Big[\tilde{f}(S_i'+1)  - \tilde{f}(S_i') \Big] = \mathbb{E}_{S_i'}\Big[h(S'_i) \Big] \, ,
    \end{align*}
    is strictly monotone decreasing in $p_i$. The intuition behind this latter statement is that as we increase the parameter $p_i$ in the $\text{Binom}(n-1, p_i)$ distribution, probability mass is shifted to the right, and since $h$ is strictly decreasing, the expectation itself must also be strictly monotone decreasing. To make this rigorous, we observe that when $p_i < p'_i$, we have $\text{Binom}(n-1, p_i)$ is strictly stochastically dominated by $\text{Binom}(n-1, p'_i)$. We combine this observation with the general result that if a distribution $\mu$ is strictly stochastically dominated by another distribution $\mu'$, then for any strictly decreasing function $g : \mathbb{R} \rightarrow \mathbb{R}$, we have
    \begin{align*}
        \mathbb{E}_{Z \sim \mu}[g(Z)] > \mathbb{E}_{Z \sim \mu'}[g(Z)] \, \;
    \end{align*}
    see, for example, \citet{levy2006stochastic} for further background.
    Therefore, given the component of the Lagrangian in Lemma~\ref{lem:unif-reward-general-f}, we require all $p_i$ equal in order for all derivatives of the Lagrangian to be 0, and hence the constrained optimum is attained with all $p_i$ equal, as required.
\end{proof}

\section{Additional Results}

\subsection{Objective Landscape of \methodnameshort+Max}
\label{app:objective-landscape-slice-rosa-max}
\begin{wrapfigure}{r}{0.37\textwidth}
    \vspace{-6em}
    \centering
    \includegraphics[width=\linewidth]{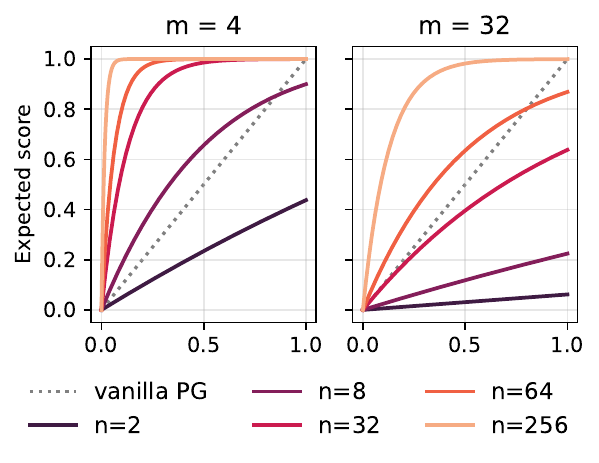}
    \caption{Objective landscape slice from zero reward to the optimal, maximally diverse policy, given $m$ reward functions and $n$ sampled actions.}
    \label{fig:rosa-max-objective-slice}
    \vspace{-3em}
\end{wrapfigure}

As discussed in relation to Proposition~\ref{prop:optimal-binary-distinct}, the action sampling parameter $n$ does \textit{not} affect the optimal policy, meaning that for a uniform reward function distribution, \textit{any} $n\geq 2$ induces a maximally diverse policy over all reward functions, thus the global optimum can be optimized with small parallel sampling budgets (e.g. just $n=2$ i.i.d.\ action samples from $\pi_\theta$). However, in practice we expect the selection of this parameter to be important to algorithmic performance. 

To see this, consider the case of $m$ uniformly distributed reward functions giving binary reward to $m$ distinct actions. The optimal solution is to put $1/m$ mass uniformly over the $m$ good actions. Now, consider a straight line going from all negative actions (at $t=0$) to this optimum (at $t=1$),
\begin{equation}
    1 - \bigg(1-\frac{t}{m}\bigg)^n \,.
\end{equation}
This can be thought of as a slice of the objective landscape. We plot this in Figure~\ref{fig:rosa-max-objective-slice}. Observe that while all slices have maximum at $t=1$, they follow different curvatures, with $n \gg m$ resulting in the objective flattening out too early which can slow optimization due to the shallow landscape. 

\textbf{Expected advantage of \methodnameshort and vanilla PG.} Another way to conceptualize the above trade-off is to consider the ``\textit{effective advantage}'' term in Equation~\eqref{eq:rosa-max-loo-pg} (i.e. $\max_{i\in\{1,\ldots,n\}} R(Y_i) \, - \max_{i\neq j} R(Y_i)$), which multiplies the gradient of the log-probability. This term becomes 0 as soon as there are two optimal $Y$'s in the sampled action set $\setY$. This can be illustrated clearly by plotting the \textit{expected effective advantage} for the vanilla PG ($f_{\text{mean}}$) compared against that of the set max function ($f_{\text{max}}$) in a binary reward setting (Figure~\ref{fig:expected-advantage-binary}, left). Observe that the effective advantage of both correct and incorrect actions quickly decay and converge to 0 as the policy improves for \methodnameshort+Max, while the gap between the two is constant for vanilla PG. Fundamentally, we do want training to slow down as we approach a better policy (as \methodnameshort+Max induces), though not so much that it slows learning. Other set functions, such as the \methodnameshort+Softmax, can help induce different optimization landscapes. 

\begin{figure}[h]
    \centering
    \begin{subfigure}[b]{0.39\linewidth}
        \centering
        \includegraphics[width=\linewidth]{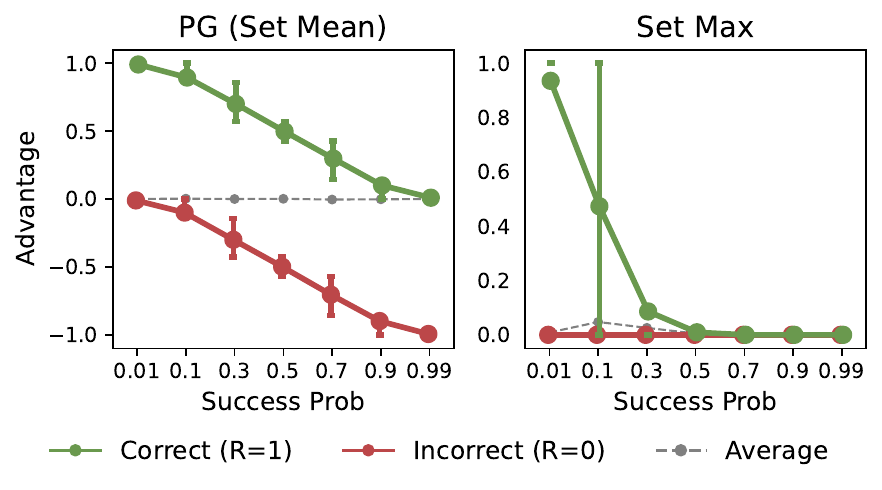}
    \end{subfigure}
    \hfill
    \begin{subfigure}[b]{0.58\linewidth}
        \centering
        \includegraphics[width=\linewidth]{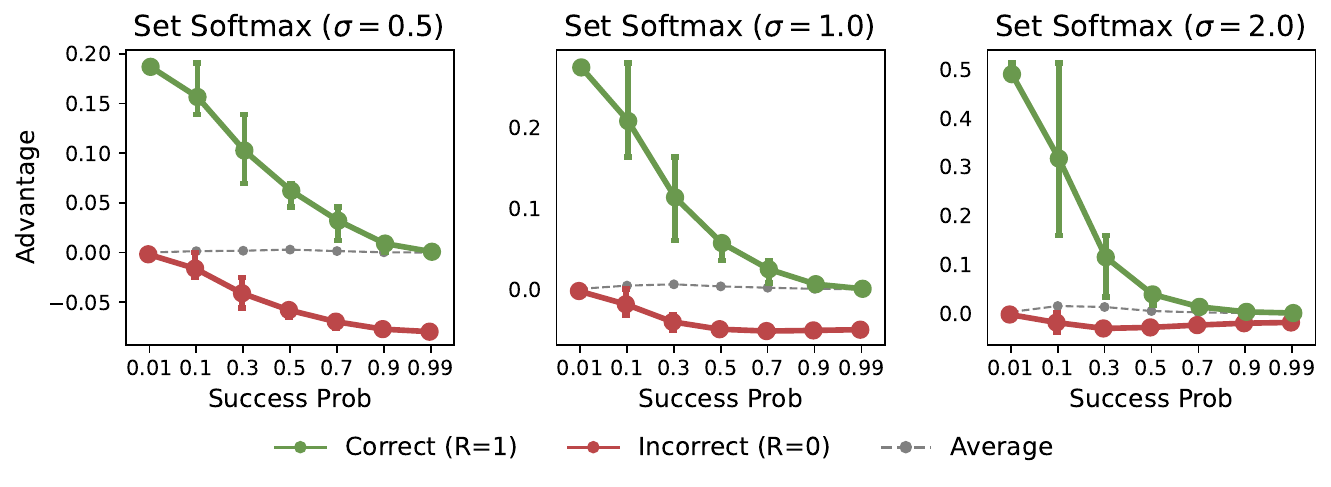}
    \end{subfigure}
    \caption{Expected advantage of a positive (green) and negative (red) sample in a binary reward task, as function of policy performance (x axis), for set size $n=8$, and different set functions with leave-one-out baseline. Error bar denotes the middle 50\% percentile interval of data. We see the \textit{difference} in expected advantage between positive and negative samples (i) is constant for \textbf{vanilla PG}, and (ii) quickly decays to zero for \textbf{Set Max}. Moreover, \textbf{Set Softmax} interpolates between these two behaviours depending on the choice of inverse temperature parameter $\sigma$.}
    \label{fig:expected-advantage-binary}
\end{figure}

\section{Additional discussion of related work}\label{app:related-work}

In this section, we provide more technical discussion with related subfields of RL.

\subsection{Risk-sensitive reinforcement learning}\label{app:risk-sensitive-rl}

The set max function used in \methodnameshort+Max can be interpreted as optimizing a function of the reward distribution that is \emph{not} expressible as an expected utility, but rather as a \emph{distortion risk measure} \citep{yaari1987dual,wang1996premium}. The particular distributional property is summarized below \citep{wang1996premium}.

\begin{lemma}\label{lemma:set-max-risk-distortion}
    The expected max-of-n reward can be expressed through a distorted expectation,
    \begin{align}\label{eq:distortion}
        \mathbb{E}_{Y_1, \dots , Y_n \iidsim  \pi}\big[
            \max(R(Y_1),\ldots,R(Y_n))
        \big] = \int_{0}^1 F^{-1}(u^{1/n}) \,  \mathrm{d}u \, , 
    \end{align}
    where $F$ is the CDF of $R(Y)$, $Y \sim \pi$.
\end{lemma}
\begin{proof}
    Recall that $R(Y)$ is equal in distribution to $F^{-1}(U)$, where $U \sim \text{Uniform}([0,1])$, and $F^{-1}$ is the quantile function of $R(Y)$. We then have that $(R(Y_1)\,\ldots,R(Y_n))$ is equal in distribution to $(F^{-1}(U_1),\ldots,F^{-1}(U_n))$, where $(U_i)_{i=1}^n \overset{\text{i.i.d.}}{\sim} \text{Uniform}([0,1])$. It follows that,
    \begin{align*}
        \max(F^{-1}(U_1),\ldots,F^{-1}(U_n)) = F^{-1}(\max(U_1,\ldots,U_n)) \, ,
    \end{align*}
    due to monotonicity of $F^{-1}$. Now, $\max(U_1,\ldots,U_n)$ is distributed equally to $U_1^{1/n}$, since
    \begin{align*}
        \mathbb{P}(\max(U_1,\ldots,U_n) \leq z) = \mathbb{P}( \forall i  \, , U_i \leq z) = \mathbb{P}(U_1 \leq z)^n = z^n = \mathbb{P}(U_1 \leq z^n) = \mathbb{P}(U_1^{1/n} \leq z) \, . 
    \end{align*}
    Hence, we have
    \begin{align*}
        \mathbb{E}_{Y_{1:n} \overset{i.i.d.}{\sim} \pi}[\max(R(Y_1),\ldots,R(Y_n))]  = \mathbb{E}_{U \sim \text{Uniform}([0,1])}[F^{-1}(U^{1/n})] \, ,
    \end{align*}
    as required.
\end{proof}
Note that without the $1/n$ exponent, the expression on the right-hand side of Eq.~\eqref{eq:distortion} would evaluate to the mean of the distribution, $\mathbb{E}_{Y\sim \pi} [R(Y)]$. The integral in Lemma~\ref{lemma:set-max-risk-distortion} therefore reports an asymmetric summary of the distribution by \emph{distorting} the variable of integration.

\begin{figure}[h]
    \centering
    \begin{subfigure}[b]{0.31\linewidth}
        \centering
        \includegraphics[width=\linewidth]{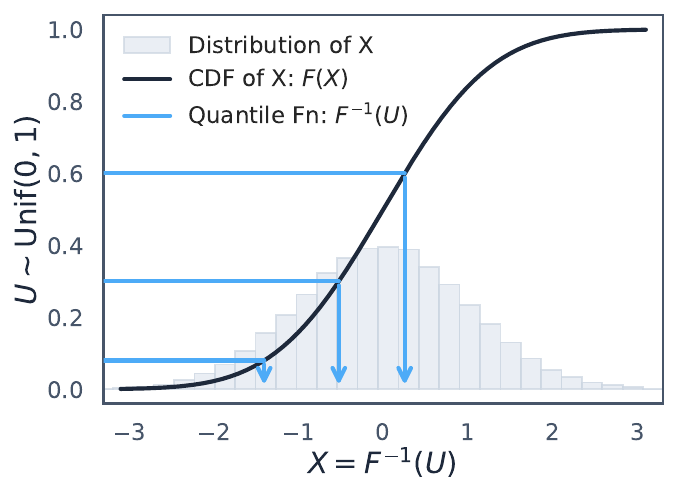}
        \caption{The Quantile transform ($F^{-1}$) commutes with set max operation}
        \label{fig:risk-distortion-quantile}
    \end{subfigure}
    \qquad
    \begin{subfigure}[b]{0.56\linewidth}
        \centering
        \includegraphics[width=\linewidth]{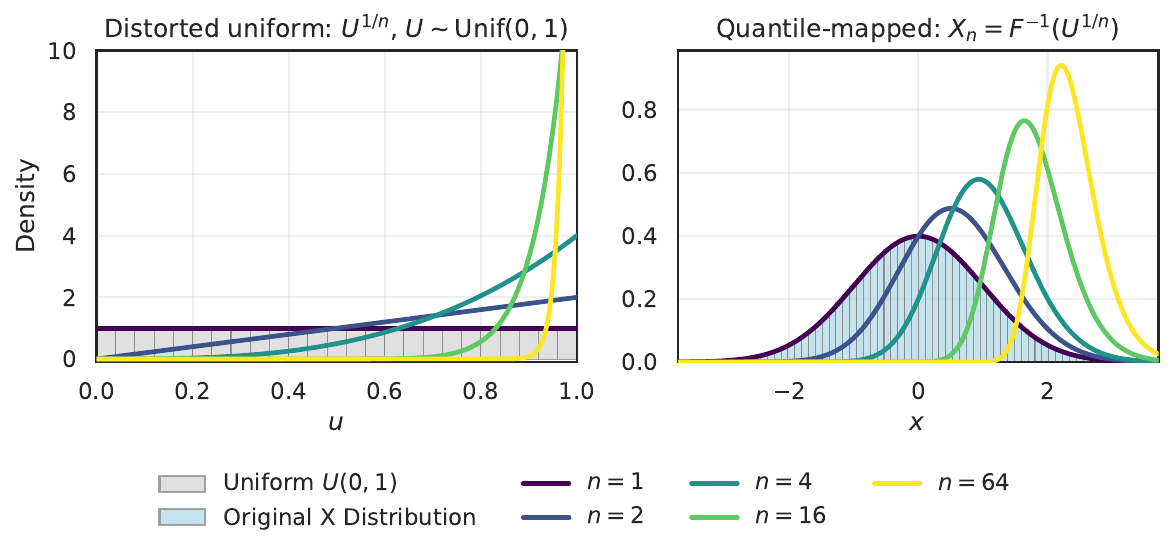}
        \caption{Max-of-n (right) as transforms of distorted uniform (left)}
        \label{fig:risk-distortion-transforms}
    \end{subfigure}
    \caption{Max-of-$n$ as an estimator of a distortion risk measure. (\textbf{a}) The quantile transform allows for sampling of any 1D random variable $X$ by mapping a uniform random variable through its quantile function (i.e., the inverse CDF). Because this transformation commutes with order-preserving operations like the maximum, a max-of-$n$ distribution can be identically modelled by taking the max-of-$n$ of uniform variables before applying the quantile transform. (\textbf{b}) Consequently, the max-of-$n$ distribution of $X$ can be viewed as a quantile transform of a distorted uniform distribution. Specifically, we sample $U \sim \text{Unif}(0,1)$, apply the distortion $U^{1/n}$, and pass it to the inverse CDF to obtain $F^{-1}(U^{1/n})$.}
    \label{fig:set-max-as-risk-distortion}
\end{figure}

Figure~\ref{fig:set-max-as-risk-distortion} further illustrates the intuition of Lemma~\ref{lemma:set-max-risk-distortion}, where the max-of-$n$ operation can be viewed as a quantile transform of a distorted uniform distribution. Specifically, transforming a standard uniform sample $U \sim \text{Unif}(0,1)$ into $U^{1/n}$ yields a $\text{Beta}(n,1)$ distribution, which is identically distributed to the maximum of $n$ independent uniforms. More generally, evaluating expectations under a distorted probability space defines a \emph{distortion risk measure} \citep{yaari1987dual,wang1996premium}. This framework is widely used to model risk sensitivity, such as in cumulative prospect theory \citep{tversky1992advances}, and extensively in risk-aware reinforcement learning \citep{chow2014cvar,dabney2018implicit,cho2023pitfall,coache2026robust}.
All in all, the implicit use of distributional properties beyond expected utilities, such as the distortion risk measure above, is key to the \methodnameshort framework being able to express objectives with the desired optimal policies.

\paragraph{Differences between \methodnameshort and Distributional RL.} 
We highlight an important distinction between the sources of randomness in \methodnameshort and those in distributional RL \citep{chung1987discounted,bellemare2023distributional} and risk-sensitive RL. Specifically, there are \emph{two} distinct sources of randomness in \methodnameshort: randomness over sampled outputs (``aleatoric'') and randomness over reward functions (``epistemic''). Distributional RL and risk-sensitive RL traditionally focus on aleatoric randomness \citep[Section~3.3]{bellemare2017distributional,dabney2018implicit}. By modeling the return distribution, distributional RL can optimize risk-sensitive quantities beyond the mean, such as up-weighting high-return outcomes over low-return ones \citep{dabney2018implicit}. The \emph{set function} in \methodnameshort operate over this same aleatoric noise: as established in Lemma~\ref{lemma:set-max-risk-distortion}, the max-of-$N$ operator yields an risk-sensitive summary of the output-induced reward distribution, serving as a Monte Carlo estimator of the same risk-sensitive objectives studied in distributional RL. \methodnameshort then \emph{additionally} considers distributions over \emph{reward functions}---a perspective traditionally emphasized in Bayesian approaches such as PSRL \citep{osband2013psrl}. 
It is precisely this combination of an optimistic aleatoric summary over epistemic reward function uncertainty that enables \methodnameshort to produce a diverse policy, allocating action probabilities in a manner calibrated to reward function uncertainty.

\subsection{Multi-objective reinforcement learning}\label{app:morl}

In the main paper, we mention comparison with multi-objective reinforcement learning \citep{roijers2013survey,hayes2022practical}. Problem formulations in this domain typically focus on finding good policies in settings with a finite collection of reward functions $R_1,\ldots,R_d$. We discuss comparisons to scalarization approaches in the bandit setting below, which are the most closely related variety of MORL problems to our setting, though often in the literature these approaches are expressed more generally in the MDP setting.

\subsubsection{Scalarization approaches}\label{app:morl-scalarization}

The principal approaches to handling multiple reward functions---and those closely related to the \methodnameshort framework---find optimal policies with respect to a scalar objective constructed from the individual rewards.
The first approach is termed \emph{expected scalarized return} (ESR), and selects a \emph{scalarization} function $s : \mathbb{R}^d \rightarrow \mathbb{R}$ to obtain an objective
\begin{align}\label{eq:esr}
    \mathbb{E}_{Y \sim \pi}[s(R_1(Y), \ldots, R_d(Y))] \, .
\end{align}
The second approach is termed \emph{scalarized expected return} (SER), and instead inverts the ordering of the expectation and scalarization function, leading to an objective of the form
\begin{align}\label{eq:ser}
    s(\mathbb{E}_{Y \sim \pi}[R_1(Y)], \ldots, \mathbb{E}_{Y \sim \pi}[R_d(Y)]) \, .
\end{align}

Both approaches have been applied to a variety of problems in reinforcement learning that are naturally expressed through multiple reward functions. For the problems we are concerned with in this paper, however, these approaches have several drawbacks which the ROSA framework circumvents.

\textbf{Lack of expressiveness.} ESR deals solely with \emph{expected utilities}; properties of distributions that can be expressed as the expectation of a (possibly non-linear function) applied to a random variable. This is a limited class of distributional properties, and is not expressive enough to capture the notions of diversity that we consider in this paper. As a concrete example, we consider the case where $R_1,\ldots,R_d$ are one-hot over distinct actions, and would like an objective that is uniquely maximized by the policy which is uniform over this set of actions. The only possible values that the vector $(R_1(Y), \ldots, R_d(Y))$ can take on are the zero-vector, and one-hot vectors. By symmetry, our non-linearity $s$ must assign equal utility to each one-hot vector. But then the objective simply reduces to the amount of probability mass allocated to the set of actions that are optimal for one of the $(R_i)_{i=1}^d$, with no incentive for uniformity of the distribution on this set. On the other hand, SER captures quantities that are functions of expected rewards, and thus cannot capture distributional properties of the random variable $R_i(Y)$ that require more than the mean, such as the ones discussed in Section~\ref{app:risk-sensitive-rl}.

\textbf{Implementation complexity.} SER additionally encounters complexities in implementation, owing to the objective applying non-linearities to expectations of rewards; the naive approach to forming a policy gradient for this objective induces bias, requiring more intricate algorithm design, such as two-timescale approaches \citep{bai2022joint,agarwal2022multi,ganesh2026breaking}. By contrast, since the ROSA framework expresses the objective in terms of expectations over set actions, the non-linear function of the reward distribution is made implicit, and unbiased policy gradients for the objective are straightforward to compute.

\textbf{Accessing all rewards functions.} The MORL approaches described above typically assume access to all reward functions when computing an objective estimator, and thus scale computationally with $d$, the number of reward functions specified in the problem. In particular, they do not generally handle infinitely-many reward functions. By contrast, as the ROSA objective is phrased with an outer expectation over the reward function, unbiased estimates of objective and gradient can be calculated straightforwardly with samples from the reward function distribution.

\textbf{Standard policy gradient on marginalized reward function as a special case.} Given a distribution $\rho$ over a finite number of reward functions $(r_i)_{i=1}^m$, both the ESR and ROSA objectives have standard policy gradient on the reward function $\sum_{i=1}^m \rho(r_i) r_i$ as a special case. For ESR, given a distribution this is obtained by taking $s : \mathbb{R}^d \rightarrow \mathbb{R}$ to be
\begin{align*}
    s(z_1, \ldots, z_n) = \sum_{i=1}^n \rho(r_i) z_i \, .
\end{align*}
For ROSA, this is obtained for any $n$ by taking the set function $f$ to be the mean function. In the context of Figure~\ref{fig:rosa-applications} in the main paper, illustrating the general ROSA framework, this is mathematically equivalent to collapsing the plot along the reward function axis and applying a standard policy gradient. By contrast, the general ROSA framework applies a non-linear transformation, based on the policy, to each reward function before performing the reduction along this axis.

\subsubsection{Pareto optimization}\label{app:morl-pareto}

While not explored in this work, we highlight a concurrent work which explores randomized objectives for Pareto optimization \citep{bahlous2026vector}. Concretely, the authors consider a sequence of vector \emph{utility functions}, $\mathbf{u} = (u_1 , \dots , u_d)$, where $\mathbf{u}(x,y) \in \mathbb{R}^d$ is the $d$-dimensional utility feedback for inputs $(x, y)$, and aim to find a set of actions which are \emph{Pareto-optimal} for the vector utilities \citep{ehrgott2005multicriteria,miettinen1999nonlinear}. Note that in our discussion here, we have intentionally used the term \emph{utility functions}, $u_i : \mathcal{X} \times \mathcal{Y} \rightarrow \mathbb{R}$, to distinguish from the reward functions, $R: \mathcal{X} \times \mathcal{Y} \rightarrow \mathbb{R}$ in the ROSA framework. The primary objective in \citet{bahlous2026vector} computes the score for some action multiset $(y_1 , \dots , y_n)$ as,
\begin{equation}\label{eq:vpo-objective}
    \mathbb{E}_{\mathbf{w} \sim \text{Dir}(\mathbf{1})} \bigg[
        \max_{1 \leq i \leq n} \mathbf{w}^\top \mathbf{u}(x, y_i)
    \bigg] \,,
\end{equation}
where $\text{Dir}(\mathbf{1})$ is a uniform distribution over $d$-dimensional vectors that sum to 1, which linearly combines the $d$ dimensional utilities.

We note that this is in fact a special case of \methodnameshort+Max (Eq.~\eqref{eq:rosa-max-criterion}), where the randomized reward function is defined via the following reparameterization,
\begin{equation}
    R(x,y) = \mathbf{w}^\top \mathbf{u}(x,y) , \qquad \mathbf{w} \sim \text{Dir}(\mathbf{1}) \,.
\end{equation}
Substituting the above distribution $R \sim \rho$ into the \methodnameshort+Max criterion (Eq.~\eqref{eq:rosa-max-criterion}) recovers Equation~\eqref{eq:vpo-objective}. Thus, since the primary objective in \citet{bahlous2026vector} is an instance of \methodnameshort, the theoretical understanding we provide in this work can be applied to understand their work as well. From this perspective, \citet{bahlous2026vector} elegantly apply \emph{distributions over scalarizations} (of utilities $\mathbf{u}$) to induce the reward-function distribution, and demonstrate an additional application of the \methodnameshort framework: to do Pareto-optimization.

It is worth noting that the approach of \citet{bahlous2026vector} differs operationally from standard action-set methods. Rather than sampling the action multiset i.i.d.\ in parallel---as done in \citet{tang2025optimizing,hamid2025polychromic,orney2026poly} and our work---their action multiset is generated \emph{sequentially} within a single rollout. Sequential generation allows later action samples to condition on earlier ones and may, in principle, enjoy greater representation flexibility and improved multiset diversity.\footnote{However, whether sequential sampling actually helps in current LLMs is debated, see \citet{gu2026parallel}.} Yet, it imposes a number of additional restrictions: (i) prompts have to be modified to specify multiset format and size, (ii) the model must have sufficient instruction-following ability and generate following a rigidly parsable multi-action schema, and (iii) using large multisets requires long autoregressive generations. Conversely, parallel sampling requires no structural assumptions on the prompt and generation beyond compatibility with a reward function, while better exploiting modern parallelized hardware. 
Credit assignment also differs operationally for the two approaches: because the sequential approach optimizes a single concatenated generation containing multiple actions, its current operationalization lacks action-specific credit assignment. In contrast, parallel sampling naturally admits action-level variance-reduction techniques (Appendix~\ref{app:rosa-policy-gradient}).

Ultimately, the generalized \methodnameshort formulation remains orthogonal and complementary to the choice of scalarization distributions, which are seamlessly absorbed into the reward function distribution $\rho$. Identifying the right reward function distribution for Pareto-diversity stands as a promising direction for future research. For instance, \citet{bahlous2026vector} uses linear scalarization, while non-linear scalarization methods can have different coverage properties that better suit Pareto geometries \citep{ehrgott2005multicriteria,bowman1976relationship,steuer1983interactive,lin2024smooth}. Promisingly, any such exploration can be viewed as \methodnameshort with distinct choices of reward distributions, and stand to benefit directly from the theoretical understanding we have established in the current work.

\subsection{Diversity bonus approaches}
\label{app:related-work-diversity-bonus}

While \methodnameshort induces diversity as naturally emergent from the underlying reward distribution, a separate, distinct paradigm enforces output diversity by pairing a single reward function $r$ with an explicit, external diversity mechanism. Intuitively, these approaches manually alter the optimization objective to ``boost'' distinct actions and penalize redundant ones. In this section, we describe their properties and draw contrasts with the \methodnameshort framework.

\paragraph{Per-sample multiplicative diversity. } A common approach is to scale the reward of each individual action by its average pairwise distance to all other concurrent actions in the multiset \citep{li2025jointly}. The modified per-sample reward $\bar{r}(Y_i)$ is (again dropping state dependence to lighten notation),
\begin{equation}\label{eq:darling}
    \bar{r}(Y_i) = r(Y_i)  \cdot \bigg(
        \frac{1}{n-1} \sum_{1 \leq j \leq n ,\, j \neq i} d(Y_i, Y_j) 
    \bigg) \,,
\end{equation}
where $d: \mathcal{Y} \times \mathcal{Y} \rightarrow \mathbb{R}$ is a pairwise diversity function quantifying the distance between a pair of actions. For instance, $d$ can be a binary classifier encoding whether $y$ and $y'$ are semantically equivalent \citep{zhang2025noveltybench,li2025jointly}.

While intuitive, this approach suffers from two primary pitfalls. First, optimizing Equation~\eqref{eq:darling} via standard policy gradient (Eq.~\eqref{eq:reinforce}) introduces \emph{gradient bias}. This is because an individual action’s modified reward dynamically depends on all other sampled actions, and standard PG ignores cross-derivatives tracking how changing one action's probability affects the diversity scores of concurrent actions. Second, even if the set policy gradient is correctly computed, per-sample multiplicative diversity warps the optimization landscape such that a worse policy (in terms of reward) can be ranked much higher than a nearly optimal policy with low diversity, therefore complicating policy improvement. This can be observed in Figure~\ref{fig:simplex-multiplicative-divs-per-sample-mult}: a policy which assign 50-50 probability to actions \texttt{A} and \texttt{C} ($\mathbb{E}[R(Y)] \approx 0.5$) is preferred over one that puts nearly all probabilities on \texttt{A} ($\mathbb{E}[R(Y)] \approx 1$).

\begin{figure}[h]
    \centering
    \begin{subfigure}[b]{0.23\linewidth}
        \centering
        \includegraphics[width=\linewidth]{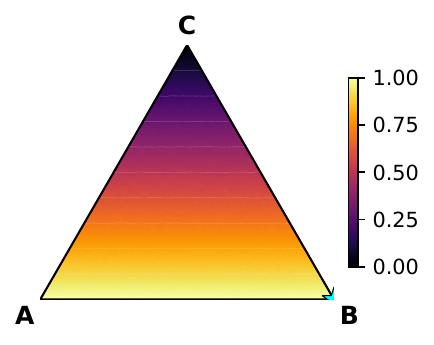}
        \caption{Reward only}
    \end{subfigure}
    \quad
    \begin{subfigure}[b]{0.23\linewidth}
        \centering
        \includegraphics[width=\linewidth]{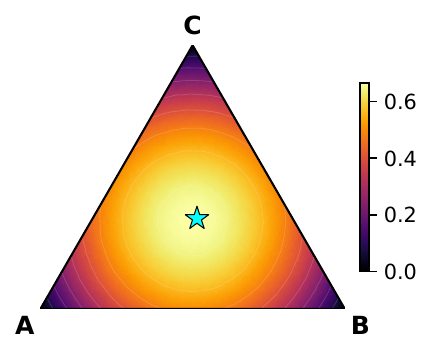}
        \caption{Diversity only}
    \end{subfigure}
    \quad 
    \begin{subfigure}[b]{0.23\linewidth}
        \centering
        \includegraphics[width=\linewidth]{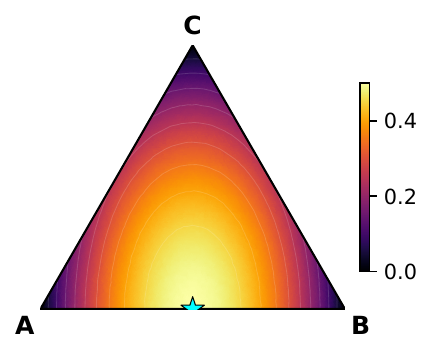}
        \caption{Mult. diversity}
        \label{fig:simplex-multiplicative-divs-per-sample-mult}
    \end{subfigure}

    \caption{Per-sample multiplicative diversity. We numerically simulate the average policy score using (a) reward function only, (b) pairwise diversity function only with diversity function $d(y, y') = \indf{ y \not= y' }$, and (c) per-sample multiplicative between the reward and diversity. Actions \texttt{A} and \texttt{B} receive a reward of 1, and \texttt{C} receives 0. Colours denote the averaged score for different policy on the 3-category simplex. Details in Appendix~\ref{app:objective-simplex}. }
    \label{fig:simplex-multiplicative-divs}
\end{figure}

\paragraph{Set-level multiplicative diversity.} An alternative approach applies a global diversity modifier directly to the collective action set reward \citep{hamid2025polychromic,orney2026poly}. This alters the total multiset objective to the form:
\begin{equation}\label{eq:set-mult-div}
    \bar{R}(Y_{1:n}) = \bigg(
        \frac{1}{n} \sum_{i=1}^n r(Y_i) 
    \bigg) \cdot D(Y_{1:n}) \,,
\end{equation}
where $D(Y_{1:n})$ is a global set-level metric. A number of diversity metrics are available here, which we again study in the three-category simplex to visualize the average score of Eq.~\eqref{eq:set-mult-div} with different action probabilities (Figure~\ref{fig:simplex-set-multiplicative-divs}). One intuitive set diversity metric is the \emph{set average distance} (i.e. set version of Equation~\eqref{eq:darling}),
\begin{align*}
    D_{\text{AvgDist}}(Y_{1:n}) = \frac{1}{n(n-1)} \sum\nolimits_{i=1}^n \sum\nolimits_{1 \leq j \leq n ,\, j \neq i} d(Y_i, Y_j)  \,,
\end{align*}
which we visualize in Figure~\ref{fig:set-multiplicative-avg-distance}. Another set diversity metric is the \emph{number of distinct clusters} \citep{hamid2025polychromic,orney2026poly},
\begin{align*}
    D_{\text{Clusters}}(Y_{1:n}) = \frac{\text{number of distinct clusters in $Y_{1:n}$}}{n} \,,
\end{align*}
which we compute as the number of unique actions, visualized in Figure~\ref{fig:set-multiplicative-unique_actions}. Both approaches inherit the landscape-distortion pitfall of multiplicative scaling, and have optimal policies that are reward sub-optimal, as the objective prefers highly diverse sets of incorrect actions over redundant sets of correct actions. In contrast, \methodnameshort+Max (Fig.~\ref{fig:rosa-max-uniform-simplex}) and \methodnameshort+Softmax (Fig.~\ref{fig:rosa-softmax-uniform-simplex}) naturally preserve sub-optimal policy ordering, ensuring that policy improvement gradients always point in a direction that simultaneously maximizes correctness and diversity.

\begin{figure}[h]
    \centering
    \begin{subfigure}[b]{0.23\linewidth}
        \centering
        \includegraphics[width=\linewidth]{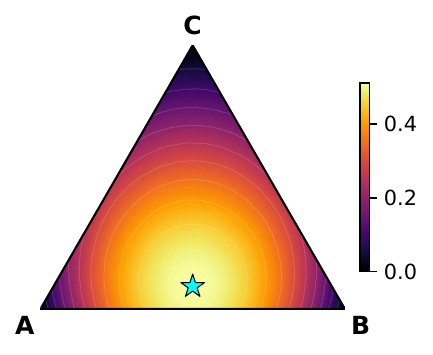}
        \caption{Average distance}
        \label{fig:set-multiplicative-avg-distance}
    \end{subfigure}
    \qquad
    \begin{subfigure}[b]{0.23\linewidth}
        \centering
        \includegraphics[width=\linewidth]{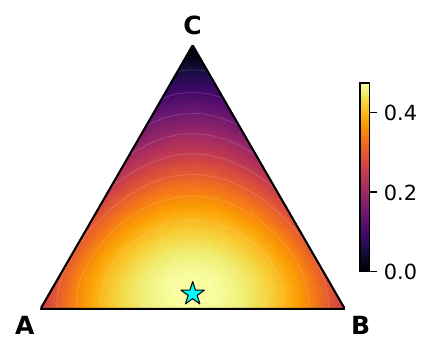}
        \caption{Unique actions}
        \label{fig:set-multiplicative-unique_actions}
    \end{subfigure}

    \caption{Set multiplicative diversity. We numerically simulate the average policy score using average reward multiplied with (a) average set distance, and (b) number of distinct actions. Actions \texttt{A} and \texttt{B} receive a reward of 1, and \texttt{C} receives 0. Colours denote the averaged score for different policy on the 3-category simplex. Details in Appendix~\ref{app:objective-simplex}. }
    \label{fig:simplex-set-multiplicative-divs}
\end{figure}

\paragraph{Additive diversity. } 
Finally, we also study \textit{additive} bonuses in Figure~\ref{fig:simplex-additive-divs} of the form,
\begin{align*}
    \bar{r}_{\text{additive}}(Y_i) = r(Y_i)  + \alpha  \bigg(
        \frac{1}{n-1} \sum_{1 \leq j \leq n ,\, j \neq i} \indf{Y_i \neq Y_j} 
    \bigg) \,.
\end{align*}
While additive formulations can mitigate some of the policy-ordering distortions found in multiplicative methods, they introduce a separate practical pitfall: extreme sensitivity to the scale of the reward and diversity bonuses. The hyperparameter $\alpha$ requires extensive tuning: a value too small may fail to induce diversity, while a value too large similarly results in low-reward optimal policies.

\begin{figure}[h]
    \centering
    \begin{subfigure}[b]{0.23\linewidth}
        \centering
        \includegraphics[width=\linewidth]{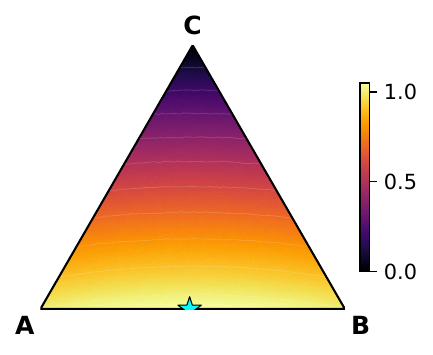}
        \caption{\centering$\alpha=0.1$}
    \end{subfigure}
    \begin{subfigure}[b]{0.23\linewidth}
        \centering
        \includegraphics[width=\linewidth]{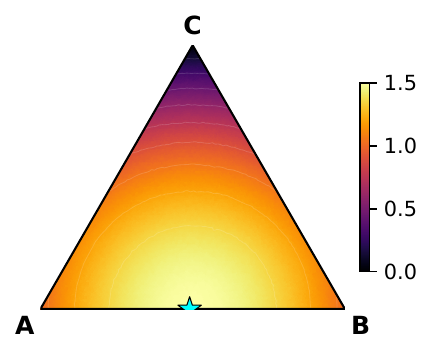}
        \caption{\centering$\alpha=1$}
    \end{subfigure}
    \begin{subfigure}[b]{0.23\linewidth}
        \centering
        \includegraphics[width=\linewidth]{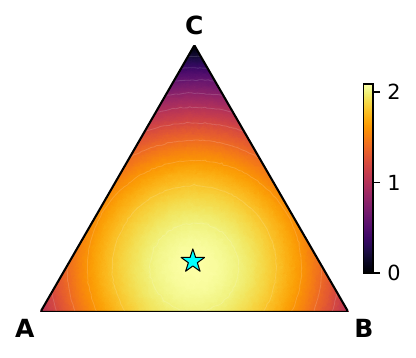}
        \caption{\centering$\alpha=2$}
    \end{subfigure}
    
    \caption{Numerical simulation of additively combined reward and diversity scores. Actions \texttt{A} and \texttt{B} receive a reward of 1, while \texttt{C} receives 0. Colours denote the average action multiset score for different policy distributions on the simplex.}
    \label{fig:simplex-additive-divs}
\end{figure}

\paragraph{Interpreting diversity bonuses under the \methodnameshort framework.}
Interestingly, in the case of binary rewards, the multiplicative diversity bonus (Eq.~\eqref{eq:darling}) can be interpreted under the \methodnameshort framework, albeit with a set function not supported by the theoretical guarantees of Theorem~\ref{thm:optimal-binary-distinct}. Specifically, we will assume that actions rewarded under different reward functions are different. Suppose we have taken a multiset of $n$ actions with $k$ of them being rewarded under reward function $r$, the sum of quality $\times$ diversity scores for this set of correct actions is,
\begin{align*}
    k \times \frac{n-k}{n-1} \,.
\end{align*}
The multiplicative term is because the other $n-k$ actions are different from the $k$ rewarded actions, and therefore contribute to the diversity bonus. In other words, multiplicative diversity can be realized under the \methodnameshort framework with a success-count reward function $\tilde{f}(k) = k(n-k)/(n-1)$. We plot this in Figure~\ref{fig:tildef}, alongside the success-count function for ROSA+Max, and ROSA+Softmax. Note that for ROSA+Softmax, the success-count function can be calculated by summing over the $k$ contributions of a reward-1 action in Equation~\eqref{eq:softmax-set-function}, to obtain,
\begin{align*}
    \frac{ke}{ke + (n-k)} \, .
\end{align*}
\begin{wrapfigure}{r}{0.4\textwidth}
    \vspace{-60pt}
    \centering
    \includegraphics[width=\linewidth]{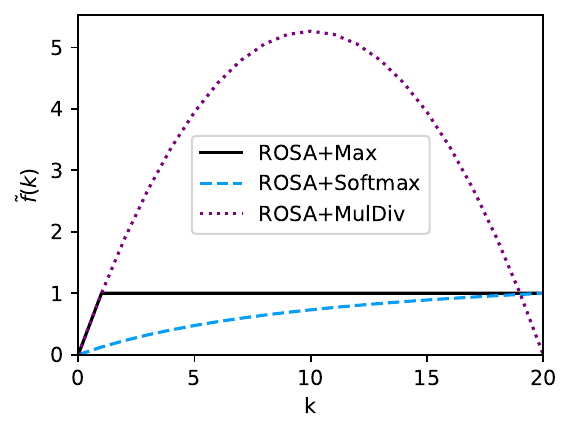}
    \caption{Success-count functions $\tilde{f}$ for ROSA+Max, ROSA+Softmax, and the multiplicative diversity objective appearing in Equation~\eqref{eq:darling}, in the binary reward setting.}
    \label{fig:tildef}
    \vspace{-10pt}
\end{wrapfigure}
Interestingly, the set function corresponding to multiplicative diversity does not satisfy the criteria of Theorem~\ref{thm:optimal-binary-distinct}; in particular, monotonicity is not satisfied. This can manifest certain issues: if only one unique correct action is observed ($k$ times) in a collection of $n$ actions, values of $k$ around $n/2$ actually receive a higher score than values of $k$ closer to $n$. This forces optimization to prefer a policy that has low expected reward ($\sim k/n$) in order to be diverse, and is an instance where the diversity-modified objective distorts the ordering of sub-optimal policies. On the other hand, we know from Theorem~\ref{thm:optimal-binary-distinct} that there exists a family of $\tilde{f}$ that provides both high expected reward and a diverse optimal policy.

We conclude this section with Table~\ref{tab:desiderata-comparison}, which summarizes, in the context of our central problem, some key properties of ROSA and a variety of comparator methods discussed in this section.

\begin{table}[htbp]
    \centering
    \vspace{0.5em}
    \begin{tabular}{@{} l c c c c c @{}}
        \toprule
        \textbf{Desiderata} & \makecell{\textbf{Entropy}\\ \textbf{Regularization}} & \makecell{\textbf{Diversity}\\ \textbf{Bonus}} & \makecell{\textbf{MORL}\\ \textbf{(ESR)}}& \makecell{\textbf{MORL} \\ \textbf{(SER)}} & \textbf{\methodnameshort} \\
        \midrule
        Preserves reward-maximizing optimum      & \redcross & \greentick / \redcross & \redcross & \redcross & \greentick \\
        Ranks sub-optimal policies correctly     & \redcross & \redcross & \redcross & \redcross & \greentick \\
        Has stochastic optimal policies (diverse) & \greentick & \greentick & \redcross & \greentick & \greentick \\
        Unbiased gradient estimation  & \greentick & \greentick & \greentick & \redcross & \greentick \\
        Handles arbitrary reward distributions      & N/A & N/A & \redcross & \redcross & \greentick \\
        \bottomrule
    \end{tabular}
    \caption{Comparison of \methodnameshort to other methods in diversity-promoting and multi-objective RL. \methodnameshort is the only approach that satisfies all desiderata: (i) its objective optimum aligns with an expected reward optimum, (ii) it correctly orders sub-optimal policies by their rewards, (iii) it maintains a stochastic, diversity-preserving optimal policy, (iv) it admits straightforward optimization via unbiased gradient estimators from action samples $Y \sim \pi(\cdot | X)$, and (v) it naturally accommodates arbitrary distributions over potentially infinite reward functions. Note that depending on the type of diversity bonus approach, it may or may not have a reward-maximizing optimal policy (see Appendix~\ref{app:related-work-diversity-bonus}).}
    \label{tab:desiderata-comparison}
\end{table}

\section{Experimental details}

In this section, we collect together additional experimental details and results, to aid reproduction and intuition regarding the core method of the paper and associated hyperparameter selection.

\subsection{Global landscape}
\label{app:objective-simplex}

To generate Figure~\ref{fig:rosa-uniform-simplex}, we considered categorical distributions over the three-action simplex, $\pi = (\pi_A,\pi_B,\pi_C) \in \Delta^2$. For each value of $\pi$, we drew $n=4$ independent actions $Y_1,\dots,Y_n \sim \mathrm{Categorical}(\pi)$ and evaluated a set-level score. We estimated the expected score at each $\pi$ by averaging up to 100,000 independent trials. The simplex landscape was obtained by evaluating this Monte Carlo estimate over approximately $5000$ values of $\pi$ tiled across the simplex.

There are three possible actions, $Y_i \in \{A,B,C\}$, and two binary reward
functions:
\begin{equation}
    R_A(Y) = \indf{Y = A}, \qquad R_B(Y) = \indf{Y = B}.
\end{equation}
Let $m=2$ denote the number of reward functions and write
$Y_{1:n} = (Y_1,\dots,Y_n)$. By default the reward distribution is $\rho(R_A) = \rho(R_B) = 0.5$ unless otherwise specified (in Figure~\ref{fig:rosa-max-nonuniform-simplex}). 
The standard policy-gradient-style score sums rewards across sampled actions:
\begin{equation}
    f_{\mathrm{PG}}(Y_{1:n}) = \sum_{k \in \{A,B\}} \rho(R_k) \sum_{i=1}^{n}  R_k(Y_i).
\end{equation}
The entropy-regularized score adds the entropy of the sampling distribution:
\begin{equation}
    f_{\mathrm{PG+Ent}}(Y_{1:n}, \pi) = \sum_{k \in \{A,B\}} \sum_{i=1}^{n} \Big[ R_k(Y_i) + \mathcal{H}(\pi) \Big] \,.
\end{equation}
The max-based set score rewards whether each reward function is achieved at
least once in the sampled set:
\begin{equation}
    f_{\text{\methodnameshort+Max}}(Y_{1:n}) =
    \sum_{k \in \{A,B\}} \rho(R_k) \max_{1 \leq i \leq n} R_k(Y_i).
\end{equation}
The diversity-weighted score weights each rewarded action by the fraction of
other samples in the set that differ from it \citep{li2025jointly}:
\begin{equation}
    f_{\text{Div}}(Y_{1:n}) = \sum_{k \in \{A,B\}} \sum_{i=1}^{n} R_k(Y_i) \cdot
    \Big(
        \frac{1}{n-1}  \sum_{j \neq i} \indf{Y_j \neq Y_i} 
    \Big).
\end{equation}

In Figure~\ref{fig:landscapes}, we further illustrate optimization landscapes in a simple example with three actions ($A$, $B$, and $C$), and two one-hot rewards (on each of $A$ and $B$). The figure shows how optimization landscapes vary depending on the objective in question (standard marginalized reward function/PG, PG with entropy regularization, ROSA+Max, ROSA+Softmax), as well as the weights of each reward function in the problem. 

\begin{figure}[h]
    \centering
    \begin{subfigure}[b]{0.9\linewidth}
        \centering
        \includegraphics[width=\linewidth]{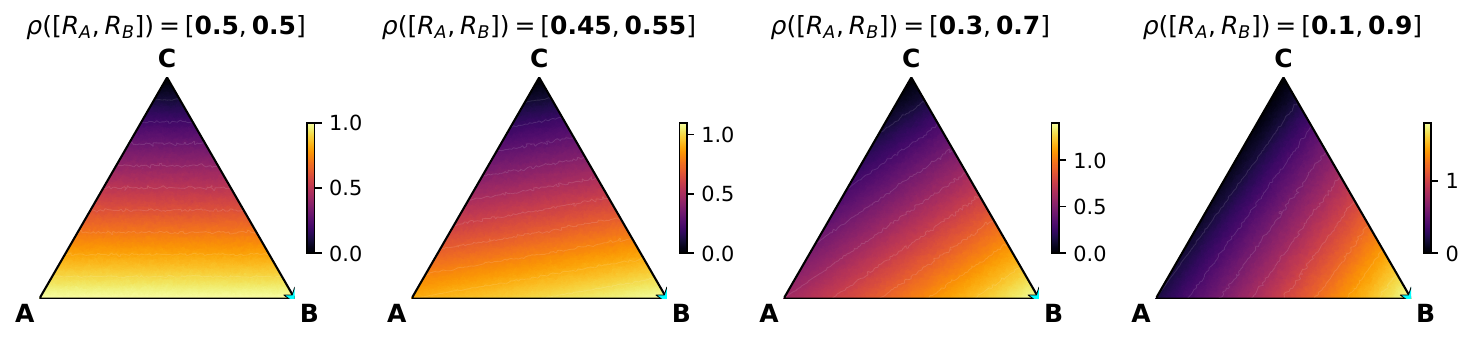}
        \caption{PG}
    \end{subfigure}

    \begin{subfigure}[b]{0.9\linewidth}
        \centering
        \includegraphics[width=\linewidth]{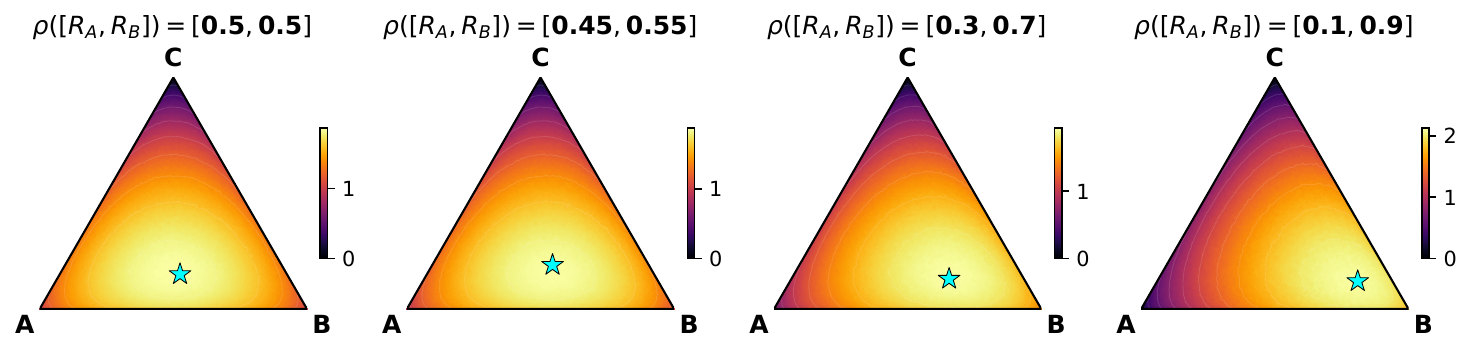}
        \caption{PG with entropy regularization}
    \end{subfigure}

    \begin{subfigure}[b]{0.9\linewidth}
        \centering
        \includegraphics[width=\linewidth]{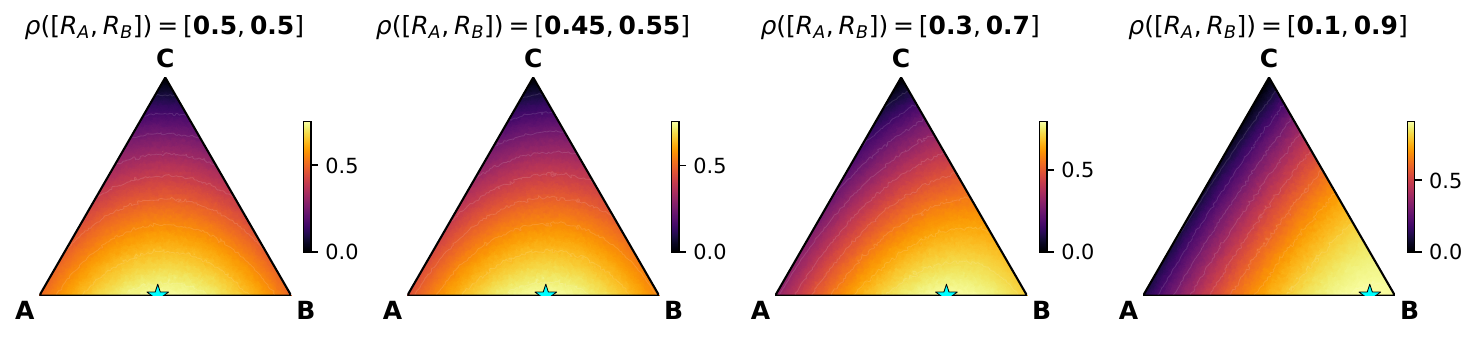}
        \caption{\methodnameshort+Max}
    \end{subfigure}

    \begin{subfigure}[b]{0.9\linewidth}
        \centering
        \includegraphics[width=\linewidth]{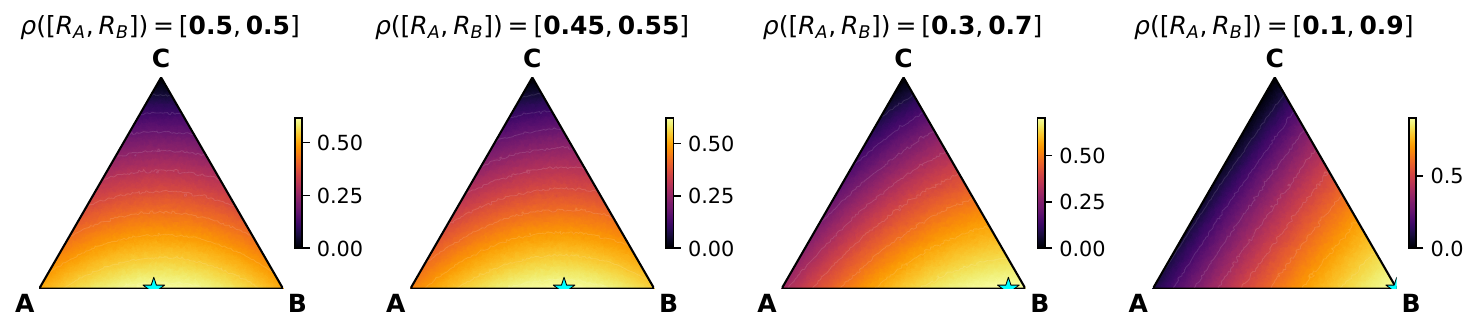}
        \caption{\methodnameshort+Softmax}
    \end{subfigure}

    \caption{Optimization landscapes for a variety of objectives and reward function distributions.}
    \label{fig:landscapes}
\end{figure}

\clearpage

\subsection{Small transformer experiment with competing preferences}
\label{app:exp-competing-pref-small-transformer}

In Figure~\ref{fig:opposing-penalty-2RFns}, we use a small transformer in which each token's vocabulary index corresponds to its integer representation. The reward function checks if the fixed length output sequence gives parity or anti-parity matched answers, 
\begin{align}
    R_{1}(X,Y_i) &= 
    \begin{cases}
    +1 & \text{if parity}(X,Y_i), \\
    -1 & \text{if anti-parity}(X,Y_i), \\
    0 & \text{otherwise}
    \end{cases}
    \qquad
    R_{2}(X,Y_i) = -\,R_{1}(X,Y_i).
\end{align}
The policy is a small causal transformer trained from scratch, in which each token's vocabulary index corresponds to its integer value. Prompts are sampled as sequences of distinct integers, and we sample fixed-size, non-overlapping training and evaluation datasets of prompts. All methods are trained with on-policy REINFORCE-style updates using Adam with a constant learning rate, and differ only in how the per-sample advantage is computed from the sampled action set (Alg.~\ref{alg:rosa-loo}). The training details are in Table~\ref{tab:exp-competing-pref-small-transformer}.

\begin{table}%
    \centering
    \small
    \begin{tabular}{ll}
        \toprule
        \textbf{Hyperparameter} & \textbf{Value} \\
        \midrule
        Transformer blocks & $3$ \\
        Embedding dimension & $64$ \\
        Attention heads & $1$ \\
        Feed-forward dimension & $128$ \\
        Activation & ReLU \\
        Positional encoding & Learned \\
        Vocabulary size & $20$ \\
        Prompt length & $3$ \\
        Total sequence length & $8$ \\
        \specialrule{0.08em}{0.6em}{0.4em}
        Optimizer & Adam \\
        Learning rate & $2 \times 10^{-4}$ (constant) \\
        Gradient clipping & None \\
        Weight decay & None \\
        Training steps & $2048$ \\
        Train dataset size & $512$ \\
        Number of prompts per batch & $256$ \\
        Samples per prompt ($n$) & $16$ \\
        Sampling temperature & $1.0$ \\
        \specialrule{0.08em}{0.6em}{0.4em}
        Eval dataset size & $512$ \\
        Eval samples per prompt & $512$ \\
        \bottomrule
    \end{tabular}
    \caption{Hyperparameters for the binary parity/anti-parity experiment.}
    \label{tab:exp-competing-pref-small-transformer}
\end{table}

\subsection{Multiple correct answers}

\label{app:exp-math-lengths}

In Figure~\ref{fig:multi-correct-math-lengths}, we fine-tune an instruction-tuned Gemma~2 2B model on the MATH training set. We use $m=4$ reward functions with uniform weights; each gives a reward of $1$ if and only if the final answer (extracted from \texttt{\textbackslash boxed\{\}}) is correct \textit{and} the response length in characters falls in the corresponding range, $[0, 400]$, $[400, 1600]$, $[1600, 4000]$, or $[4000, 40000]$. All methods share the same GRPO-style training pipeline, with multiple gradient steps per rollout batch using the clipped surrogate objective, and a KL penalty towards the initial policy; the methods differ only in the advantage calculations. We evaluate on MATH500 by sampling multiple responses per problem and reporting pass@$k$ both overall and within each length range. Training details are in Table~\ref{tab:exp-math-lengths}.

\begin{table}%
    \centering
    \small
    \begin{tabular}{ll}
        \toprule
        \textbf{Hyperparameter} & \textbf{Value} \\
        \midrule
        Base model & Gemma 2 2B (instruction-tuned) \\
        Training data & MATH train split ($7500$ problems) \\
        Max sequence length & $2048$ tokens \\
        Precision & bfloat16 \\
        \specialrule{0.08em}{0.6em}{0.4em}
        Optimizer & Adam ($\beta_1 = 0.9$, $\beta_2 = 0.95$, $\epsilon = 10^{-8}$) \\
        Learning rate & $10^{-7}$ (constant) \\
        Weight decay & $0$ \\
        Gradient norm clipping & $1.0$ \\
        Training steps & $10{,}000$ \\
        Prompts per step & $16$ \\
        Samples per prompt ($n$) & $8$ \\
        Rollout sampling temperature & $1.0$ \\
        \specialrule{0.08em}{0.6em}{0.4em}
        Gradient steps per rollout batch & $4$ \\
        PPO clipping $\epsilon$ & $0.2$ \\
        KL penalty coefficient & $0.001$ \\
        Advantage normalization & None \\
        Max absolute advantage & $10$ \\
        \bottomrule
    \end{tabular}
    \caption{Hyperparameters for the MATH experiment with multiple length-preferring reward functions.}
    \label{tab:exp-math-lengths}
\end{table}

Example generations for \methodnameshort trained policies on MATH with reward function distribution over different generations lengths are shown in Figure~\ref{fig:math-multi-len-generations}. 

\begin{figure}%
    \centering
    \tiny      %
    \ttfamily  %
    
    \begin{subfigure}{\textwidth}
        \noindent Here's how to evaluate the expression:
        
        \smallskip
        \noindent 1. \textbf{Multiply:} \\
        \indent $\bullet$ $(1 + 2i)6 = 6 + 12i$
        
        \smallskip
        \noindent 2. \textbf{Subtract:} \\
        \indent $\bullet$ $6 + 12i - 3i = 6 + 12i - 3i = 6 + 9i$ 
        
        \smallskip
        \noindent Therefore, $(1 + 2i)6 - 3i = \boxed{6 + 9i}$.
        \caption{Short Generation (length = 196)}
    \end{subfigure}
    
    \vspace{1.5em}
    \noindent\rule{\textwidth}{0.4pt} %
    \vspace{1.5em}
    
    \begin{subfigure}{\textwidth}
        \begin{multicols}{3}
        \noindent Here's how to evaluate the expression:
        
        \smallskip
        \noindent \textbf{1. Distribute:} \\
        $\bullet$ $(1 + 2i)6 = 6 + 12i$ \\
        \textbf{2. Substitute and Simplify:} \\
        $\bullet$ $6 + 12i - 3i = 6 + 12i - 3i = 6 + 9i$
        
        \smallskip
        \noindent \textbf{Therefore, $(1 + 2i)6 - 3i = \boxed{6 + 9i}$.} 
         
        \smallskip
        \noindent Let's break down what we did so you can understand:
        
        \noindent $\bullet$ \textbf{Distribute:} The distributive property (sometimes called the ``FOIL'' property) applies to multiplication. \\
        $\bullet$ \textbf{Combine Like Terms:} Since we have a sum of two terms with 'i', we combine them. 
         
        \smallskip
        \noindent Let's clarify:
        
        \noindent $\bullet$ \textbf{Complex Numbers:} An expression like $(1 + 2i)$ represents a complex number. 'i' is the imaginary unit, where $i^2 = -1$. 
        
        \smallskip
        \noindent Let's go through the steps again:
        
        \noindent 1. \textbf{Distribute:} $(1 + 2i)6 = 6 + 12i$ \\
        2. \textbf{Substituting and Simplifying:} \\
        \indent $\bullet$ $6 + 12i - 3i = 6 + 12i - 3i = 6 + 9i$
        
        \noindent $\bullet$ \textbf{Answer:} $(1 + 2i)6 - 3i = 6 + 9i$ 
        
        \smallskip
        \noindent Let's highlight why involving \textit{i} within complex numbers using imaginary units is important to understand. 
        
        \smallskip
        \noindent Let's summarize with why involving \textit{i} in complex numbers is important:
        
        \noindent $\bullet$ \textbf{Expanding Imaginary Concepts}: Complex numbers provide a foundation for advanced mathematical concepts like calculus and linear algebra which rely on the use of imaginary numbers. \\
        $\bullet$ \textbf{Solving Equations with Solutions Beyond Real Numbers:} Complex numbers are crucial for solving equations that have no real-number solutions. For example, they help model phenomena like the electrical signals in circuits. \\
        $\bullet$ \textbf{Technology}: Complex numbers have applications in fields such as signal processing, quantum mechanics, and electrical engineering.
        
        \smallskip
        \noindent Let's break it down: \\
        $\bullet$ \textbf{Real Numbers:} We use real numbers for everyday accounting, measuring length, etc. \\
        $\bullet$ \textbf{Complex Numbers:} We use complex numbers for scenarios where we need to solve equations with no real-number solutions, or to model elements of reality, like electrical signals.
        
        \smallskip
        \noindent Let's break that down further:
        
        \smallskip
        \noindent Let's simplify the problem:
        
        \noindent 1. \textbf{Distribute:} $(1 + 2i)6 = 6 + 12i$ \\
        2. \textbf{Substituting and Simplifying:} \\
        \indent $\bullet$ $6 + 12i - 3i = 6 + 12i - 3i = 6 + 9i$   
        
        \smallskip
        \noindent You can further simplify equations involving complex numbers! 
        
        \smallskip
        \noindent Let's break it down further:
        
        \noindent \textbf{1. Distribute:} $(1 + 2i)6 = 6 + 12i$
         
        \noindent \textbf{2. Simplify:} \\
        \indent $\bullet$ $6 + 12i - 3i = 6 + 9i$
        
        \smallskip
        \noindent Let's break down the main concepts:
         
        \noindent $\bullet$ \textbf{Imaginary Unit (i):} The imaginary unit, \textit{i}, is defined as the square root of -1. \\
        $\bullet$ \textbf{Complex Numbers:} Complex numbers are formed by combining real numbers and the imaginary unit. In mathematical notation, they are represented as a + bi, where 'a' and 'b' are real numbers, and 'i' is the imaginary unit.
        
        \smallskip
        \noindent Let's summarize with the importance of imaginary numbers:
         
        \noindent $\bullet$ \textbf{Expanding Imaginary Concepts:} Imaginary units provide an even deeper understanding of numbers in mathematics. \\
        $\bullet$ \textbf{Generalized Solutions:} They have applications in solving equations where traditional algebraic methods fail. \\
        $\bullet$ \textbf{Applications in Simulating Physical Phenomena:} For complex scenarios in physics, engineering, or other STEM fields.
        
        \smallskip
        \noindent Let's recap!
        
        \noindent $\bullet$ \textbf{Real numbers} describe quantities like length, height, time, temperature, and age on a usual scale. \\
        $\bullet$ \textbf{Imaginary numbers} expand to include quantities like \textit{imaginary units (i)}. They may seem strange or abstract compared to real numbers, but they have vital applications in our understanding of the world.
        
        \smallskip
        \noindent Let's break it down
         
        \noindent $\bullet$ \textbf{Operations on Complex Numbers:} Just like with algebra with real numbers, you can perform operations on complex numbers, such as addition, subtraction, multiplication and division. \\
        $\bullet$ \textbf{Imaginary Unit's Importance:} The concept of \textit{i} goes beyond merely being ``imaginary.'' You'll step up your mathematical game by understanding how these numbers interact with equations coming from the area of physics, biology, or engineering. 
        
        \smallskip
        \noindent Let's sum up!
        
        \smallskip
        \noindent Remember the key takeaways about complex numbers:
        
        \noindent $\bullet$ \textbf{Structure:} Complex numbers are thought of as having a real and imaginary part. \\
        $\bullet$ \textbf{Representation:} A complex number can be represented as \textit{a + bi} where 'a' and 'b' are real numbers, and 'i' is the imaginary unit ($\sqrt{-1}$). \\
        $\bullet$ \textbf{Applications:} They are found in electricity, magnetism, quantum mechanics, fluid dynamics, data science) and a variety of other fields. 
        \end{multicols}
        \caption{Long Generation (length = 4245)}
    \end{subfigure}
    
    \caption{Comparison of length-diverse responses from a \methodnameshort trained policy. Question: \texttt{Evaluate \$(1+2i)6-3i\$}.}
    \label{fig:math-multi-len-generations}
\end{figure}

\subsection{Uncertain reward functions}
\label{app:reward-function-uncertainty-toy}

\begin{figure}[h]
    \centering
    \begin{subfigure}[b]{0.9\linewidth}
        \centering
        \includegraphics[width=\linewidth]{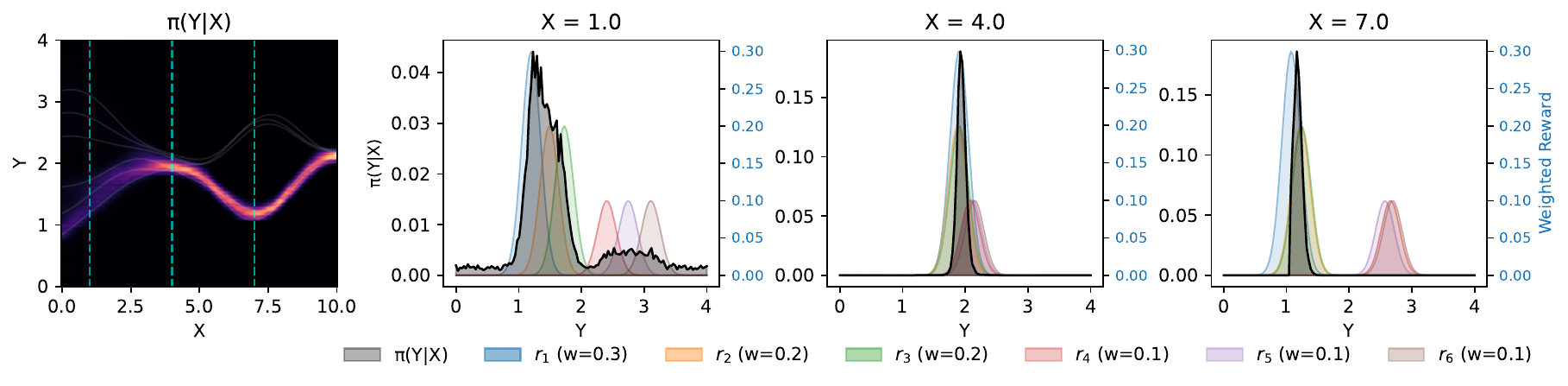}
        \caption{PG (entropy coef = 0.1)}
    \end{subfigure}

    \begin{subfigure}[b]{0.9\linewidth}
        \centering
        \includegraphics[width=\linewidth]{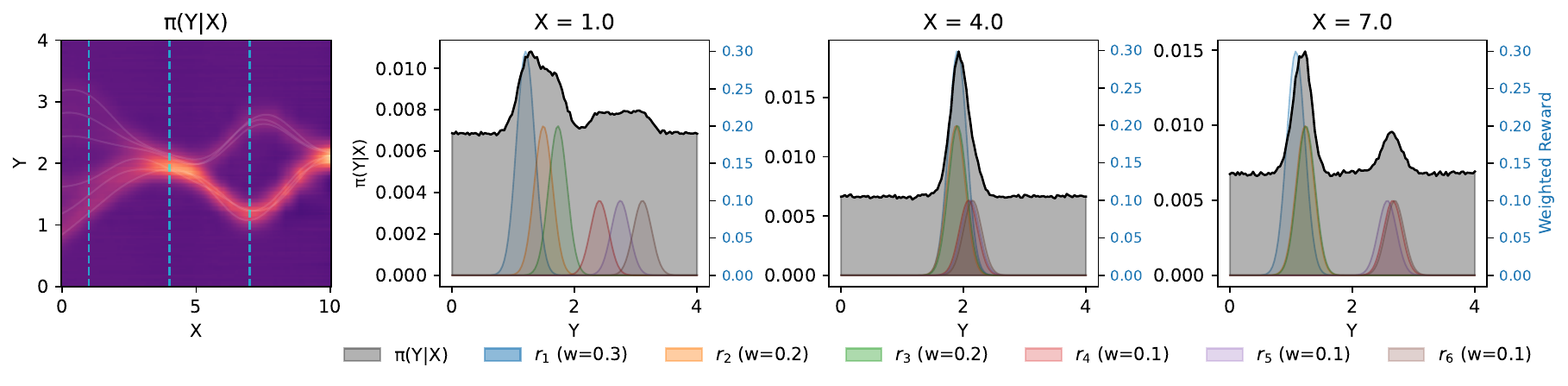}
        \caption{PG (entropy coef = 0.8)}
    \end{subfigure}

    \begin{subfigure}[b]{0.9\linewidth}
        \centering
        \includegraphics[width=\linewidth]{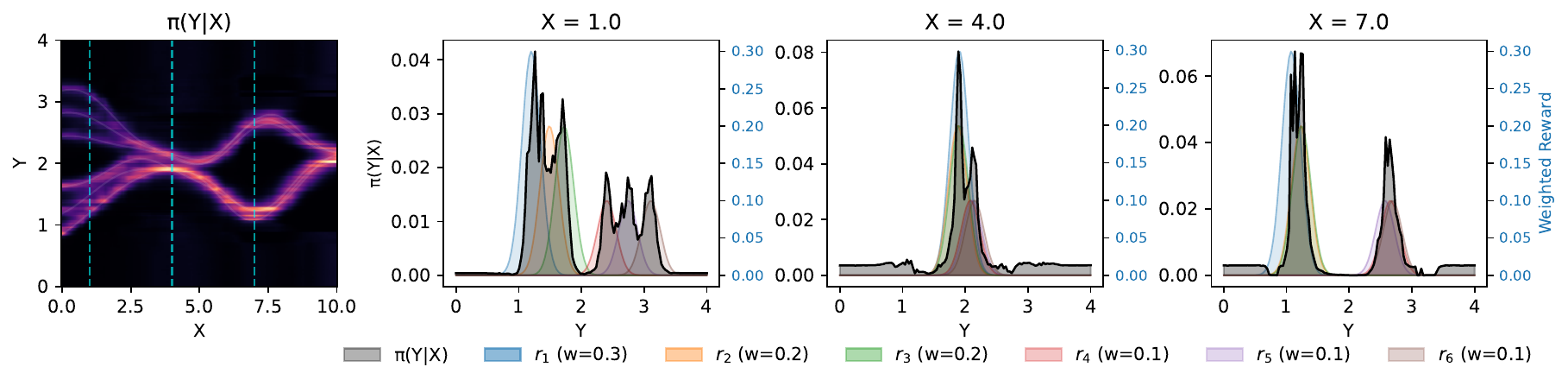}
        \caption{ROSA+Max (entropy coef = 0.01)}
    \end{subfigure}

    \caption{Policy trained on 6 non-uniformly weighted reward functions $R:\mathcal{X} \times \mathcal{Y} \rightarrow \mathbb{R}$. Left plot shows the reward functions and their maximum (white lines). Right plot shows the policy over $\mathcal{Y}$ as a function of $X$. Observe PG training leads to collapse, while PG with a high entropy coefficient converges to the dominant reward functions. On the other hand, \methodnameshort+Max covers all reward modes in proportion to their reward function weight.}
    \label{fig:uncertain-rewards-2d-extended}
\end{figure}

In Figure~\ref{fig:uncertain-rewards-2d}, we design an experiment with 6 reward functions defined over a continuous 2D domain ($x \in[0, 10]$, $y \in [0, 4]$). Each reward function is parametrized by an optimal curve $g_i(x)$, and $r_i(x, y) = \exp(-(y - g_i(x))^2 / (2\sigma^2))$. In other words, each reward function has maximum of $1$ when $y = g_i(x)$, and decaying as a Gaussian with distance from the optimum, with $\sigma = 0.15$ (see Figure~\ref{fig:uncertain-rewards-2d-extended} for the reward landscape). The optimal $y$s for each reward function do not always agree, creating regions where the reward function agree on a single optimum, and regions where they diverge into distinct modes. The aggregated reward is a weighted combination of the 6 functions with weights (0.3, 0.2, 0.2, 0.1, 0.1, 0.1). We represent the policy $\pi_\theta (Y|X)$ as a categorical distribution over 128 uniformly spaced bins covering the $Y$ domain (bin width $\approx$ 0.031), parametrized by a 2-layer ReLU MLP (64 hidden units per layer) implemented in JAX with Equinox. We optimize using Adam (setting the learning rate to $3 \times 10^{-4}$) for 100,000 steps with batch size 128 and group size 8 (i.e., 8 action samples $Y$ per context $X$ per batch element). See Figure~\ref{fig:uncertain-rewards-2d-extended} for extended figures showing cross-sections of the learned policy $\pi(\cdot | x)$ at different values of $x$.

\subsection{Noisy MATH judges}
\label{app:math-noisy-judge}

In Figure~\ref{fig:math-noisy-judge}, we construct an ensemble of reward judges in MATH. We train with \methodnameshort (Max and Softmax), vanilla PG, as well as ensemble-based methods suggested by \citet{coste2023reward}. Specifically, given $m$ rewards $(r_i)_{i=1}^m$, ensemble min optimizes the minimum $\min((r_i)_{i=1}^m)$ reward, while ensemble stddev optimizes the mean minus the standard deviation $\frac{1}{m} \sum_{i=1}^m r_i - \text{stddev}((r_{i=1}^m))$.

In Figure~\ref{fig:aime25-ood}, we directly take a \methodnameshort+Max trained checkpoint from the above MATH task, and evaluate it zero shot in AIME2025.

\newpage